\documentclass{article}
\usepackage[utf8]{inputenc} 
\usepackage[T1]{fontenc}    

\usepackage[a4paper, total={6in, 9in}]{geometry}

\usepackage{graphicx}
\usepackage{hyperref}       
\usepackage{url}            
\usepackage{booktabs}       
\usepackage{amsfonts}       
\usepackage{nicefrac}       
\usepackage{microtype}      
\usepackage[round]{natbib}
\usepackage{amsmath,amsthm,amssymb,amsfonts}
\usepackage[dvipsnames,table]{xcolor}         
\usepackage{xspace}
\usepackage{dsfont}
\usepackage{multirow}
\usepackage{makecell}
\usepackage{nicefrac}
\usepackage{enumitem}

\usepackage{stmaryrd}
\usepackage{nidanfloat}
\usepackage{makecell}
\usepackage{pifont}
\usepackage{tikz}
\usetikzlibrary{calc}
\usepackage[normalem]{ulem} 
\usepackage{wrapfig}

\definecolor{pearThree}{HTML}{E74C3C}
\definecolor{pearcomp}{HTML}{B97E29}
\definecolor{pearDark}{HTML}{2980B9}
\definecolor{pearDarker}{HTML}{1D2DEC}
\hypersetup{
	colorlinks,
	citecolor=pearDark,
	linkcolor=pearThree,
	breaklinks=true,
	urlcolor=black}
\usepackage[ruled]{algorithm2e}
\usepackage{multicol}

\newcommand\encircle[1]{%
  \tikz[baseline=(X.base)]
    \node (X) [fill=black, text=white, draw, minimum size=3.2mm, inner sep=0, outer sep=0, font=\small, rounded corners=1pt] {\strut #1};}

\newtheorem{theorem}{Theorem}
\newtheorem{lemma}[theorem]{Lemma}
\newtheorem{corollary}[theorem]{Corollary}

\newtheorem{proposition}[theorem]{Proposition}
\newtheorem{definition}[theorem]{Definition}

\newtheorem{assumption}[theorem]{Assumption}

\theoremstyle{definition}
\newtheorem{remark}{Remark}

\usepackage{array}
\newcommand{\PreserveBackslash}[1]{\let\temp=\\#1\let\\=\temp}
\newcolumntype{C}[1]{>{\PreserveBackslash\centering}p{#1}}
\newcolumntype{R}[1]{>{\PreserveBackslash\raggedleft}p{#1}}
\newcolumntype{L}[1]{>{\PreserveBackslash\raggedright}p{#1}}

\usepackage{fontawesome5}
\definecolor{reddark}{HTML}{790A1D}
\definecolor{reddeep}{HTML}{890D16}
\definecolor{redrich}{HTML}{A33038}

\definecolor{bluedark}{HTML}{00406B}
\definecolor{bluedeep}{HTML}{005178}
\definecolor{bluerich}{HTML}{00629A}

\definecolor{greendark}{HTML}{005A42}
\definecolor{greendeep}{HTML}{29673B}
\definecolor{greenrich}{HTML}{1A7639}

\definecolor{navydark}{HTML}{342071}
\definecolor{navydeep}{HTML}{3E2D86}
\definecolor{navyrich}{HTML}{4E4391}

\definecolor{orangedark}{HTML}{E29400}
\definecolor{orangedeep}{HTML}{F8AF00}
\definecolor{orangerich}{HTML}{FDC400}

\definecolor{purplepure}{HTML}{804A83}

\usepackage{amsmath,amsfonts,bm}

\def\eqref#1{equation~\ref{#1}}

\def\1{\bm{1}}

\DeclareMathAlphabet{\mathsfit}{\encodingdefault}{\sfdefault}{m}{sl}
\SetMathAlphabet{\mathsfit}{bold}{\encodingdefault}{\sfdefault}{bx}{n}

\DeclareMathOperator*{\argmax}{arg\,max}

\newcommand{\lsviucb}[1]{{\small\textsc{LSVI-UCB}}\xspace}

\newcommand{\wt}[1]{\widetilde{#1}}

\newcommand{\wh}[1]{\widehat{#1}}

\newcommand{\transp}{\mathsf{T}}

\usepackage{xspace}
\DeclareRobustCommand{\eg}{e.g.,\@\xspace}
\DeclareRobustCommand{\ie}{i.e.,\@\xspace}

\newcommand{\comment}[1]{}

\newcommand{\linucb}{{\small\textsc{LinUCB}}\xspace}

\newcommand{\algo}{{\small\textsc{SBLB}}\xspace}

\usepackage[colorinlistoftodos, textwidth=25mm, disable]{todonotes}

\renewcommand{\transp}{\top}
\renewcommand{\intercal}{\top}

\usepackage{authblk}

\title{Privacy Amplification via Shuffling \\ for Linear Contextual Bandits}

\author[1,2]{Evrard Garcelon}
\affil[1]{Meta AI}
\affil[2]{CREST, ENSAE}
\author[1]{Kamalika Chaudhuri}
\author[2]{Vianney Perchet}
\author[1]{Matteo Pirotta}
\date{}

\usepackage[toc,page,header]{appendix}
\usepackage{minitoc}

\begin{document}
\doparttoc 
\faketableofcontents 

\maketitle
\begin{abstract}
  Contextual bandit algorithms are widely used in domains where it is desirable to provide a personalized service by leveraging contextual information, that may contain sensitive information that needs to be protected. Inspired by this scenario, we study the contextual linear bandit problem with differential privacy (DP) constraints. While the literature has focused on either centralized (joint DP) or local (local DP) privacy, we consider the shuffle model of privacy and we show that is possible to achieve a privacy/utility trade-off between JDP and LDP. By leveraging shuffling from privacy and batching from bandits, we present an algorithm with regret bound $\wt{\mathcal{O}}(T^{2/3}/\varepsilon^{1/3})$, while guaranteeing both central (joint) and local privacy. Our result shows that it is possible to obtain a trade-off between JDP and LDP by leveraging the shuffle model while preserving local privacy.
\end{abstract}

\section{Introduction}\label{sec:introduction}
In a \emph{contextual bandit algorithm}, at each time $t \in [T]:=\{1,\ldots,T\}$, a learner first observes a set of features $(x_{t,a})_{a\in [K]} \subset \mathbb{R}^{d}$, selects an action $a_{t}\in [K]$ out of a set of $K$ actions, and observes a reward $r_t = r(x_{t,a_t}) + \eta_t$ where $\eta_t$ is a conditionally independent zero-mean noise ($r$ is not known beforehand). 
Consequently,
the learning algorithm has to balance exploration of the environment with exploitation of the current knowledge to maximize the cumulative reward.
The performance of the the learner is measured by the cumulative regret, which is the difference between its own cumulative reward, and the cumulative reward it would have received had it always played the best action. 
Contextual bandit algorithms have achieved great practical success, and have been used for many sensitive applications such as personalization, digital marketing, healthcare and finance~\citep[e.g.,][]{mao2020realworld, wang2021roboadvising}. With these applications in mind, the literature has started investigated privacy guarantees both in bandits~\citep[e.g.,][]{Shariff2018contextual,zheng2021locally} and in RL~\citep[e.g.,][]{vietri2020privaterl,garcelon2020local}. In this paper, we focus on privacy-preserving contextual bandits.


For a contextual bandit problem on sensitive data, we assume that a single user enters the system at time $t$, and hence the context at time $t$ is their private information. To measure privacy, we use differential privacy~\citep{dwork2006calibrating} -- a privacy definition introduced by cryptographers that has emerged as the gold standard for privacy-preserving data analysis~\citep[\eg][]{erlingsson2014rappor, dwork2014algorithmic, abowd2018us, chaudhuri2011differentially, abadi2016deep, boursier2020utilityprivacy}. The standard differential privacy framework applies to static data in a batch setting, but two extensions have been proposed to address online problems. The first is Joint Differential Privacy (JDP)~\citep[\eg][]{Shariff2018contextual}, an analogue of central differential privacy, where the users trust the bandit algorithm. JDP ensures that changing a single user's private information in the data does not change the probability of any future outcome (namely, actions taken and rewards received by any other user) by much.
\begin{definition}[Joint DP]\label{def:JDP}
    For $\varepsilon>0$ and $\delta_{0}>0$, a randomized bandit agent $\mathfrak{A}$ is $(\varepsilon, \delta_{0})$-\emph{joint differentially}
    private if for every $t\in[T]$, two sequences of users, $U = \{u_{1},\dots, u_{T}\}$ and $U'=\{u_{1}',\dots, u_{T}'\}$, that differs only for the
    $t$-th user and for all events $E\subset \mathcal{A}^{[T-1]}$ then:
    \begin{align}
        \mathbb{P}(\mathfrak{A}_{-t}(U)\in E) \leq e^{\varepsilon}\mathbb{P}(\mathfrak{A}_{-t}(U')\in E) + \delta_{0}
    \end{align}
where $\mathfrak{A}_{-t}(U)$ denotes all the outputs of algorithm $\mathfrak{A}$, \ie all actions $(a_{i})_{i\neq t}$ excluding the output of time $t$ for the sequence of users $U$.
\end{definition}
A second, stronger concept is Local Differential Privacy (LDP)~\citep[\eg][]{zheng2021locally}, where the users do not trust the bandit algorithm, and transmit only sanitized versions (using a private randomizer $\mathcal{M}$) of their contexts and rewards to the algorithm. Here, LDP ensures that user information is sanitized in such a manner that changing a single user's private value does not alter the distribution of the sanitized value by much.
\begin{definition}[Local DP]\label{def:RL-LDP}
	For any $\varepsilon\geq0$ and $\delta\geq 0$, a privacy preserving
	mechanism $\mathcal{M}$  is said to be $(\varepsilon, \delta)$-\emph{locally differential} private if and only if for all users $u, u'\in \mathcal{U}$,
	contexts/rewards $((x_u, r_{u}), (x_{u'}, r_{u'})) \in (\mathbb{R}^{d} \times \mathbb{R})^{2}$ and all $O\subset
	\{ \mathcal{M}(\mathcal{B}(0, L)\times [0, 1]) \mid u\in \mathcal{U}\}$:
	\begin{align}\label{eq:LDP_RL}
			\mathbb{P}\left( \mathcal{M}((x_u, r_{u})) \in O\right) \leq e^\varepsilon\, \mathbb{P}\left( \mathcal{M}((x_{u'}, r_{u'})) \in O\right) + \delta
	\end{align}
	where $\mathcal{B}(0, L)\times [0, 1]$ is the space of context/reward associated to user $u$.
\end{definition}
Just like the standard batch setting, while LDP offers a strong notion of privacy, its utility is often much lower.
Specifically, for contextual linear bandit algorithms, while $\varepsilon$-JDP guarantees can be obtained by paying a multiplicative factor in the regret, LDP comes with a much higher impact on the regret. In fact, \citet{zheng2021locally} have shown that $\varepsilon$-LDP regret scales with $\wt{\mathcal{O}}(T^{\nicefrac{3}{4}}/\sqrt{\varepsilon})$ instead of $\wt{\mathcal{O}}(T^{\nicefrac{1}{2}}/\sqrt{\varepsilon})$ for a $\varepsilon$-JDP algorithm (see Tab.~\ref{tab:summary.regret.privacy} for more details.)

Real applications are gradually moving away from the \emph{centralized} model of privacy, favoring the simpler and stronger notion of local privacy. This change is illustrated by the rise of on-device computation for mobile application~\citep[\eg][]{apple2017learning}. The natural question we address in this paper is:
\begin{center}
	\textit{Is it possible to design a bandit algorithm with guarantees akin to local privacy but better utility?}
\end{center}

To address this question, we consider the shuffle model of privacy~\citep[e.g.,][]{cheu2019shuffling,feldman2020hiding,chen2020distributed, balle2019privacy, erlingsson2020encode} that, in supervised learning settings, allow to achieve a trade-off between central and local DP through a shuffler. The shuffler receives users' reports and permutes them before sending them to the server.
This setting was first introduced in \cite{bittau_2017}, named the \emph{ESA} model (Encode-Shuffle-Analyze) and motivated
by the need for anonymous data collection. \citet{erlingsson2020amplification} later provided an analysis of the amplification of privacy thanks to the combined
use of shuffling and local differential privacy showing that the shuffling model of privacy
is able to strike a middle ground between the totally decentralized but somewhat sample inefficient
\emph{local} model and the centralized but more sample efficient central model of privacy.
It is currently unclear whether it is possible to achieve some form of privacy/utility trade-off between these two models in the contextual bandit setting.

\subsection{Our Contributions}
In this paper, we investigate the linear contextual bandit problem under the shuffle model of privacy,
for the first time considering this privacy model in contextual bandit. Compared to the standard shuffle model (\eg in supervised learning), there are several challenges introduced by the sequential nature of the problem. First, the shuffler is executed continuously and not only once as normally considered in supervised learning. Second, the number of samples available grows with time and depends on the decisions of the learning agent. This makes the design of the algorithm non-trivial, in particular for efficiently trading-off privacy amplification and regret.

We address these challenges in two ways. First, we
carefully design separate asynchronous batch schedules for the shuffler and the bandit algorithm (\ie{} \linucb); here, batching at the shuffler is used to ensure privacy, and not just improved regret. Second, we leverage the martingale structure of the problem to analyze these batching schedules and provide privacy guarantees on the entire sequence of outputs generated by the shuffler and bandit algorithm. We summarize our main contributions as follows (see also Tab.~\ref{tab:summary.regret.privacy}):

\begin{itemize}[noitemsep,topsep=0pt,parsep=0pt,partopsep=0pt]
	\item If there is no adversary in between the shuffler and the algorithm (\ie{} the communication channel is secure), we show that it is possible to achieve a regret bound of $\wt{\mathcal{O}}\left( dT^{2/3}/\varepsilon^{1/3}\right)$ with a fixed batch size for the shuffler and dynamic batch for the bandit algorithm.
	\item In the case of adversary in between the shuffler and the users, our algorithm achieves a regret bound of $\wt{\mathcal{O}}\left(T^{3/4}/\sqrt{\varepsilon}\right)$ with a fixed batch size for the shuffler and dynamic batch for the bandit algorithm.
\end{itemize}

\begin{table}[!htp]
	\centering
		\begin{tabular}{|c|c|c|c|}
			\hline
			\multirow{2}{*}{Algorithm} & \multirow{2}{*}{Regret Bound} & \multicolumn{2}{c|}{Privacy Model}\\
			\cline{3-4}
			&&Joint DP & Local DP\\
			\hline
			\citet{Shariff2018contextual} & $\wt{\mathcal{O}}\left(\nicefrac{T^{1/2}}{\varepsilon^{1/2}}\right)$ & $(\varepsilon, \delta)$ & N/A\\
			\citet{zheng2021locally} &$\wt{\mathcal{O}}\left(\nicefrac{T^{3/4}}{\varepsilon^{1/2}}\right)$ & $(\varepsilon, \delta)$ & $(\varepsilon, \delta)$\\
			Our Cor.~\ref{cor:ldp.regret} (\emph{LDP optimization}) & $\wt{\mathcal{O}}\left(\nicefrac{T^{3/4}}{\varepsilon^{1/2}}\right)$ & $(\frac{\varepsilon^{3/2}}{T^{1/4}}, \delta)$ & $(\varepsilon, 0)$\\
			Our Cor.~\ref{cor:regret.fully.optimized} (\emph{regret optimization})  & $\wt{\mathcal{O}}\left(\nicefrac{T^{2/3}}{\varepsilon^{1/3}}\right)$ & $(\varepsilon, \delta)$ & $(\varepsilon^{2/3}T^{1/6}, 0)$\\
			\hline
		\end{tabular}
	\caption{Regret and privacy for algorithms in \emph{linear contextual bandits} for $T \geq 1/(27\varepsilon)^4$.}
	\label{tab:summary.regret.privacy}
\end{table}

\section{Preliminaries}\label{sec:preliminaries}

We consider linear contextual bandit problems, where rewards are linearly representable in the features, i.e., for any feature vector $x_{t,a}$, it writes as $r(x_{t,a}) = \langle x_{t,a}, \theta^\star \rangle$, where $\theta^\star \in \mathbb{R}^d$ is unknown.
We do not pose any assumption on the context generating process but we rely on the following standard assumptions.
\begin{assumption}\label{assumption:boundness}
    There exist $S>0$ and $L>0$ such that $\| \theta^{\star}\|_{2}\leq S$ and, for all time $t \in [T]$, arm $a \in [K]$,
    $\| x_{t,a}\|_{2}\leq L$. Furthermore, the noisy reward is $r_{t} = \langle x_{t,a}, \theta^\star \rangle + \eta_t \in [0,1]$  with $\eta_{t}$ being $\sigma$-subGaussian for some $\sigma>0$. These parameters, $L$, $S$ and $\sigma$, are  known.
\end{assumption}

The performance of the learner $\mathfrak{A}$ over $T$ steps is measured by the regret $R_{T} = \sum_{t=1}^{T} r(x_{t,a_t^\star}) - r(s_{t,a_t})$, which represents the cumulative difference between playing the optimal action $a_{t}^{\star} = \arg\max_{a\in [K]} r(x_{t,a})$ and $a_{t}$ the action selected by the algorithm.

\subsection{Shuffle-model in Contextual Bandits}
In this section, we introduce the generic shuffle-model for contextual bandit, inspired by the ESA model. In Sec.~\ref{sec:fixbatchshuffler}, we will provide the details for instantiating it in linear contextual bandits.
In the standard shuffle model, a shuffler is introduced in between the data and the algorithm. The shuffler enables privacy amplification by permuting information of $l$ users. The larger the batch, the higher the privacy amplification but also the degradation of the utility~\citep[see \eg][]{cheu2019shuffling}, leading to some fundamental trade-off between privacy amplification and utility loss.
In online learning, we observe users sequentially and it is  natural to assume that, in order to achieve privacy amplification, the shuffler builds a batch of consecutive users before communicating with the bandit algorithm.
The bandit algorithm can then behave synchronously or asynchronously w.r.t.\ the shuffler. In other words, it can update its internal statistics with the same frequency of the shuffler or use an independent batch schedule.


More formally, the shuffle-model for contextual bandit is described by the following interaction protocol (see also Fig.~\ref{fig:suffle.interaction}). At each time $t \in [T]$,
\begin{enumerate}[noitemsep, leftmargin=16pt]
    \item[\ding{182}] A new user $x_t$ receives model information from the bandit algorithm (\eg estimated rewards and confidence intervals) that are used to \emph{locally} compute the action to play. Then, the user plays the prescribed action $a_t$ which generates the associated reward $r_t$.
    \item[\ding{183}] The user sends its own privatized version of the data $\mathcal{M}_{\mathrm{LDP}}(x_{t,a_t}, r_t)$ to the shuffler. This new data is added to the shuffler batch $B^S_{k_t} :=\bigcup_{i = t_{k^S_t}}^t \big\{\mathcal{M}_{\mathrm{LDP}}(x_{i,a_i}, r_i)\big\}$, where $k^S_t$ denotes the shuffler batch at time $t$ and $t_{k}$ is the starting time of batch $k$.
    \item[\ding{184}] The bandit algorithm queries statistics from the shuffler. If the shuffler is ready to send data (\eg enough samples has been collected for privacy amplification), it computes a statistic $u$ on a permutation of the data (\ie $u(\sigma(B^S_{k_t}))$) and sends it to the bandit algorithm. Otherwise no information is provided.
    The bandit algorithm adds the new statistic to its batch (\ie $B^A_{k^A_t} := \bigcup_{i=t_{k^A_t}}^t \big\{u(\sigma(B^S_{k^S_i}))\big\}$) and may then decide to update the model as soon as data is received (\ie synchronously) or use an independent batch schedule (\ie asynchronous).
\end{enumerate}

\begin{figure}[!htb]
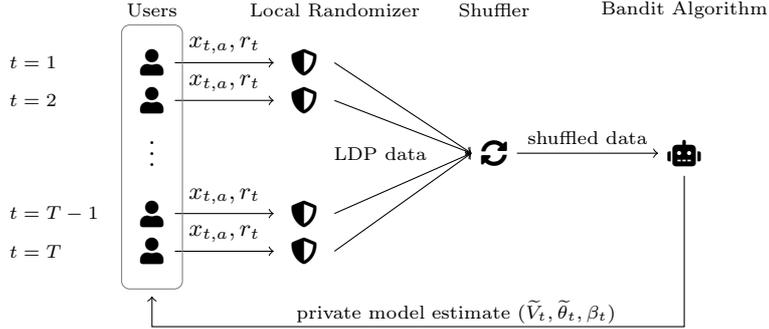

    \centering
    \tikz[]{
        \foreach \x in {0,0.5,2,2.5}
            \node at (0,\x) {\faUser};
        \node[anchor=west, font=\scriptsize] at (-2,2.5) {$t=1$};
        \node[anchor=west, font=\scriptsize] at (-2,2) {$t=2$};
        \node[anchor=west, font=\scriptsize] at (-2,0.5) {$t=T-1$};
        \node[anchor=west, font=\scriptsize] at (-2,0) {$t=T$};
        \node at (0,1.4) {$\vdots$};

        \foreach \x in {0,0.5,2,2.5} {
            \node at (2,\x) {\faShield*};
            \draw[->] (0.3,\x) -- (1.6,\x) node[midway, above, font=\small] {$x_{t,a}, r_t$};
        }

        \node (shuffler) at (4.5, 1.3) {\faSync*};
        \foreach \x in {0,0.5,2,2.5} {
            \draw[->] (2.4,\x) -- (shuffler.west);
        }
        \node[font=\scriptsize] at (3,1.3) {LDP data};

        \node (alg) at (7,1.3) {\faRobot};
        \draw[->] (shuffler.east) -- (alg.west) node[midway,above,font=\scriptsize] {shuffled data};
        \draw[->] (alg.south) |- (0, -1) -| (0,-0.6);
        \node[font=\scriptsize] at (4,-0.8) {private model estimate ($\wt{V}_t, \wt{\theta}_t, \beta_t$)};
        \draw[rounded corners=3pt, gray] (-0.4,3) rectangle (0.4, -0.5);

        \node[font=\scriptsize] at (0,3.2) {Users};
        \node[font=\scriptsize] at (2.4,3.2) {Local Randomizer};
        \node[font=\scriptsize] at (4.5,3.2) {Shuffler};
        \node[font=\scriptsize] at (7,3.2) {Bandit Algorithm};
    }
    \caption{Illustration of the shuffle model for linear contextual bandits.}
    \label{fig:suffle.interaction}
\end{figure}

The objective is to minimize the (pseudo) regret and simultaneously guarantee privacy of the data and of the statistics.
To this extent, we assume all users (including the shuffler and the bandit algorithms) behaves in an \textit{honest but curious} manner~\citep{Goldreich2009foundations}, i.e., the users and the algorithm behaves as prescribed by the protocol.
We consider different threat models for privacy, including an adversary in between \encircle{a} the user and the shuffler, \encircle{b} the shuffler and the bandit algorithm, and \encircle{c} the bandit algorithm and the user. We will show that different privacy/regret guarantees can be achieved in the different settings.

\begin{algorithm}[!htbp]
    \caption{\algo}
    \label{alg:complete.algo}
    \small
    \SetAlgoLined
    \DontPrintSemicolon
    \KwIn{LDP parameter: $\varepsilon_{0}$, privacy parameters: $\varepsilon, \delta_{0}$, regularizer: $\lambda$, context bound: $L$,
    failure probability: $\delta$, low switching parameter: $\eta$, encoding parameter: $m$, dimension: $d$, fix batch size: $\ell$}
    Initialize $j^S = j^A = 0$, $\wt{\theta}_0 = 0$, $\wt{V}_0 = \lambda I_d$ and $p = 2\big(\exp\left(\frac{2\varepsilon_{0}}{md(d+3)}\right) + 1\big)^{-1}$ \;
    \For{$t = 0, 1, \ldots$} {
        \encircle{c} \textbf{Communication with the user}\;
        User receives $\wt\theta_{j^A}$, $\wt{V}_{j^A}$ and $\beta_{j^A}$ and selects
        $
            a_t \in \argmax_{a\in [K]} \langle x_{t,a}, \wt{\theta}_{j^A}\rangle + \beta_{j^A} \| x_{t,a}\|_{\wt{V}_{j^A}^{-1}}
        $\;
        Observe reward $r_t$ and compute private statistics
        $(\wt{b}_t, \wt{w}_t )= \mathcal{M}_{\mathrm{LDP}}(x_{t,a_t},r_t, L, \varepsilon_{0}, m)$ (Alg.~\ref{alg:esa_alg})\;
        \encircle{a} \textbf{Communication with the shuffler}\;
        $B^S_{j^S} = B^S_{j^S} \cup (\wt{b}_t, \wt{w}_t)$\;
        \If{$|B^S_{j^S}| = l$}{
            Set $t_{j^S+1} = t$, compute a permutation $\sigma$ of $\llbracket t_{j^S} + 1, t_{j^S+1} \rrbracket$ and compute aggregate statistics
            \[
                \forall i\leq d, k\leq i, \qquad Z_{j^S, i} = \sum_{n=1}^{l}\sum_{q=1}^{m} \tilde{b}_{\sigma(n),i,q}
                ~~\text{ and } ~~
                U_{j^S, i, k} = \sum_{n=1}^l \sum_{q=1}^{m} \wt{w}_{\sigma(n),i,k,q}
            \]\;
            \vspace{-16pt}
            Set $U_{j^S, i, k} = U_{j^S, k, i}$, $B_{j^S+1} = \emptyset$ and $j^S = j^S + 1$\;
            \encircle{b} \textbf{Communication with the bandit algorithm}\;
            Receives $(Z_{j^S-1}, U_{j^S-1})$ and compute candidate statistics
            \begin{align*}
                \wt{B}_{j^A+1}
                &= \wt{B}_{j^A+1} + \frac{Z_{j^S-1}}{m(1-p)} - \frac{l^{S}}{2(1-p)}\\
                \wt{V}_{j^A+1}
                &= \wt{V}_{j^A+1} + \frac{U_{j^S-1}}{m(1 - p)} - \frac{l^{S}}{2(1-p)} + 2(\lambda_{j^A+1} - \lambda_{j^A})I_d
            \end{align*}
            \If{$\text{det}(\wt{V}_{j^A+1}) \geq (1 + \eta)\text{det}(\wt{V}_{j^A})$}{
                Compute $\tilde{\theta}_{j^A+1} = \frac{1}{L}\wt{V}_{j^A+1}^{-1}\wt{B}_{j^A+1}$\;
                Set $t_{j^A+1} = t$, $\beta_{j^A+1}$ and $\lambda_{j^A+1}$ as in Eq.~\eqref{eq:confidence_width} and Eq.~\eqref{eq:lambda_reg}\;
                Set $j^A = j^A+1$, $\wt{B}_{j^A+1} = \wt{B}_{j^A}$ and $\wt{V}_{j^A+1} = \wt{V}_{j^A}$\;
            }
        }
    }
\end{algorithm}

\section{Shuffle Model with Fixed-Batch Shuffler} \label{sec:fixbatchshuffler}
In this section, we provide an instantiation of the shuffle model for linear contextual bandit.
We base our algorithm on the non-private low-switching \linucb~\citep{abbasi2011improved}, that incrementally builds an estimate $\wh{\theta}_j$ of the unknown parameter $\theta^\star$.
Since the algorithm leverages sum of statistics received from the users, we consider the binary sum mechanism inspired by~\citep{cheu2019shuffling} as building block for achieving privacy in the shuffle model. While this scheme allows us to obtain standard LDP guarantees on users information, the shuffler is responsible to provide privacy amplification via batching and shuffling. The main challenge is to combine these elements with the low-switching scheme of \linucb. As we will explain later, adaptive batching at the level of \linucb is not for computational efficiency but it is rather fundamental for obtaining a good privacy/regret trade-off.

\subsection{Algorithmic Design}
In this section, we provide a full description of the Shuffle-Batched Linear Bandit (\algo) algorithm. Intuitively, the algorithm relies on a shuffler with fixed batch size to achieve privacy amplification from LDP data, and a variation of \linucb with dynamic batch schedule based on the determinant condition. The pseudo-code is reported in Alg.~\ref{alg:complete.algo}.

\paragraph{\ding{182}~Action Selection.} At each time $t$, the user $x_t$ receives, from the bandit algorithm, an estimate of the model composed by a parameter $\wt{\theta}_{k^A_t} \in \mathbb{R}^d$, a design matrix $\wt{V}_{k^A_t} \in \mathbb{R}^{d\times d}$ and confidence width $\beta_{k^A_t}$. Notice that these are parameters computed at the beginning of the batch $k^A_t$ of the bandit algorithm.
Then, the action is selected by maximizing the following standard optimistic problem:
\[
    a_t \in \argmax_{a\in [K]} \left\{ \langle x_{t,a}, \wt{\theta}_{k^A_t}\rangle + \beta_{k^A_t} \| x_{t,a}\|_{\wt{V}_{k^A_t}^{-1}} \right\}
\]
where $\beta_t$ is the size of the confidence ellipsoid, defined in Lem.~\ref{lem:confidence_ellipsoid}, which  roughly scales as $\wt{O}\Big(t_{k_{t}^{A}}^{1/4}\Big)$
. Note that it is possible to directly access the features $x_{t,a}$ of the user since this computation happens locally.
The action is played and a reward $r_t$ is observed.

\paragraph{\ding{183}~Local Privacy and Shuffler.}
Users' information is then protected through a local private mechanism $\mathcal{M}_{\mathrm{LDP}}$.
As noticed in~\citep{Shariff2018contextual},   only the information required by the algorithm, to compute $\wt{\theta}$ through ridge regression and the associated confidence interval, must be privatized.
We are thus interested in privatizing the quantities $x_{t,a_t} r_t$ and $x_{t,a_t}x_{t,a_t}^\transp$.
To obtain LDP quantities, we leverage a variation of the private mechanism introduce by~\citet{cheu2019shuffling}. We independently privatize each component of the vector $x_{t,a_t} r_t$
 and of the upper triangular part of the matrix $x_{t,a_t}x_{t,a_t}^\transp$,
the rest follows from the symmetric structure.
Each entry is normalized to $[0,1]$ and approximated by a
truncated 0/1-bit representation, which length is controlled
by the parameter $m \in \mathbb{N}^*$. The full procedure is
reported in Alg.~\ref{alg:esa_alg}.

The shuffler receives the privatize data $\mathcal{M}_{\mathrm{LDP}}(x_{t,a_t}, r_t)$ and adds it to the current batch.
The role of the shuffler is to provide additional privacy by sending data in a random order compared what it has received. At a high-level this provide an additional privacy guarantee because it breaks the link between a given user and its data. Indeed for an algorithm receiving data from the shuffler, the $t$-th row of data has little chance to come from user $t$.
If the shuffler has access to a batch of size $l$, it can provide a privacy amplification of level $l^{-1/2}$~\citep[see e.g.,][Thm. 5.4]{cheu2019shuffling}.
Ideally, we would like to shuffle all the data at each time $t$,
achieving a privacy amplification of $t^{-1/2}$.
However, this approach would not provide enough privacy
due to the fact an adversary would have multiple observations
of the same data, thus greatly decreasing the advantage of using the
 shuffling mechanism. To avoid this issue, we need to force
 the shuffler to use batches and discard samples after each batch.
  Let's denote by $l^S$ the fix batch size of the shuffler.
At time $t$, if the batch $B^S_{k_t}$ is of size $l^S$, the shuffler
permutes the data and compute the statistics required by the bandit
algorithm. {\color{black} To compute those statistics, the shuffler uses a secure and trusted third-party different that the shuffler. This third-party is assumed to be secure with for example the use of encrypted communication between the shuffler and it, like in \citep{cheu2019shuffling}}. When $|B^S_{k_t}| < l^{S}$, the shuffler do not provide any information to the bandit algorithm.
The shuffling setting is not fundamentally different than the LDP one, but it allows to achieve a large gain in privacy in the high data regime from multiple users. Shuffling allows to achieve better privacy guarantees and, overall, it improves the standard LDP protocol with virtually no cost.


\paragraph{\ding{184}~Model Estimation (the bandit algorithm).}
As last step, the bandit algorithm queries new data to the shuffler which replies only if the batch is full. If no data is received, the bandit algorithm does nothing. Otherwise, the bandit algorithm receives summary statistics $Z_{k^S_t}$ and $U_{k^S_t}$ corresponding to the sum over the shuffled batch $B^S_{k^S_t}$ of the LDP data associated to $xr$ and $xx^\transp$.
The algorithm could behave synchronously with the batch schedule of the shuffler and update the model by updating the design matrix $\wt{V}_{k^S_t+1}$ and parameter $\wt{\theta}_{k^S_t+1}$. However, this behavior would lead to a worse privacy/regret trade-off than an asynchronous data-adaptive schedule. Although it is possible to achieve the same regret bound in non-private settings with static and dynamic batch schedules, in the private case it is no more the case
because of required inflation of the confidence intervals
by a factor $t^{1/4}$ to deal with concentrations of private statistics.
In App.~\ref{app:schedule_update_alg}, we provide a more formal support to this claim.

As a consequence, we shall leverage the determinant-based condition introduced by~\citet{abbasi2011improved}.
Upon receiving the data at time $t$, the bandit algorithm has access to the following set of private statistics $\Big\{ (Z_{i}, U_i), i \in [k^S_t] \Big\}$, which is further divided into batches of various lengths.
Denote by $j = k^A_{t}$ the bandit batch at time $t$ with associated parameters $\wt{V}_{j}$, $\wt{B}_{j}$ and $\wt{\theta}_{j}$ computed at the beginning of the batch. Then, we denote by $\wt{V}_{t}$
the new design matrix obtained by updating the matrix $\wt{V}_{j}$ with all the statistics received from the shuffler after $t_j$. If $\det(\wt{V}_{t}) \geq (1+\eta) \wt{V}_{j}$, then a new batch is started and the model is updated,
 \ie $\wt{\theta}_{j+1}= \frac{1}{L}\wt{V}_{j+1}^{-1}\wt{B}_{j+1}$ is computed through ridge-regression. {\color{black} In a LinUCB fashion, the last step for
 the algorithm is to compute the size of a confidence intervals around $\wt{\theta}_{j+1}$ containing the true parameter $\theta^{\star}$. Contrary to the non-private setting~\citep{abbasi2011improved}, the algorithm
 uses wider confidence intervals to account for the noise added to ensure privacy. This increase is quite significant as the confidence intervals grow
 at a $t^{1/4}$ rate compared to $\log(t)$ in the non private setting.} Refer to Lem.~\ref{lem:confidence_ellipsoid} for the explicit definition.

\section{Analysis of The Shuffle Model with Fixed-Batch Shuffler}\label{sec:algo1.analysis}
In this section, we provide the privacy and regret guarantees of \algo. We first begin to describe which privacy guarantees are attainable in the different attack scenarios outlined in the introduction. Then we show how the regret of \algo is impacted by the these attack models.

For sake of clarity, we recall the parameters that regulates the privacy/regret analysis of our algorithm.
The first parameter $\varepsilon_{0}$ regulates the level of local differential privacy introduced by the local randomizer $\mathcal{M}_{\text{LDP}}$. However, to
simplify the analysis, we often use the alternative parameter $p := 2\big(\exp\left(\frac{2\varepsilon_{0}}{md(d+3)}\right) + 1\big)^{-1}$ derived from $\epsilon_0$ (see Alg.~\ref{alg:esa_alg}). The other two parameters $(\varepsilon, \delta_{0})$
controls the level of joint differential privacy that \algo should attain.

\subsection{Privacy Analysis of \algo}

  As discussed in Sec.~\ref{sec:preliminaries}, the shuffling model encompasses all the multiple scenarios in which the privacy of users can be threatened. 

  \vspace{.15in}
  \noindent
  \encircle{a} \textbf{Compromised communication between the user and the shuffler.}
  In the first and most harmful scenario, the communication between the
  users and the shuffler is not secured and the data can be observed by an adversary. This is the standard LDP setting in linear contextual bandit. In this case, the use of the local randomizer $\mathcal{M}_{\text{LDP}}$ guarantees
  that the data sent by the user to the shuffler are $\varepsilon_{0}$-LDP. That is to say the most stringent privacy guarantees in the differential privacy model.

\begin{proposition}[LDP guarantee]\label{prop:ldp_guarantee}
    For any $\varepsilon_{0}> 0$ and $m\in \mathbb{N}^{\star}$, $\mathcal{M}_{\text{LDP}}(.,., \varepsilon_{0}, L, m)$  is $\varepsilon_{0}$-LDP. 
\end{proposition}

    This particular scenario corresponds to a decentralized setting where the users do not trust the algorithm or the communication channel between them to
    be secure and they have to protect the privacy of their data at a individual level, that is to say to guarantee that the data sent could have been sent
    by anyone else. This setting (\ie the ``pure'' LDP scenario) is also the one studied in \citep{zheng2021locally}. We will show that we can recover their result when we want to guarantee the highest level of LDP privacy. However, at the cost of sacrificing a portion of LDP level, we can obtain a better regret bound, closing the gap with the less stringent JDP setting.


    \vspace{.15in}
  \noindent
  \encircle{b} \textbf{Compromised communication between the shuffler and the bandit algorithm.}
    In another privacy loss scenario,  an adversary can observe the same data as the bandit algorithm. Stated otherwise, the adversary has access to the output of the shuffler. In that case, \algo is still $\varepsilon_{0}$-LDP but  stronger differential privacy guarantees can be achieved thanks to privacy amplification. In this scenario, the adversary observes the different outputs of the shuffler, that are statistics computed on a number of different users.
    The question, in the differential privacy setting, is whether it is possible to know that one particular user (\ie user's data) was involved in the computation of those statistics.


    \citet{tenenbaum2021shuffle} studies a weaker version of this question in the multi-armed bandit setting where an adversary \emph{only observes the output of the shuffler for one time step}, while we focus on the more challenging case where the adversary observes all the history. Technically, this is the same difference
    as ensuring event-level privacy in the continual observation model compared to a differential privacy on a single query.
    Note that it would be possible to obtain a better regret bound if we consider the adversary model in~\citep{tenenbaum2021shuffle} since a smaller level of privacy is required (see Remark~\ref{rem:weaker.privacy.model}).

    The complicated aspect is to guarantee that the whole sequence of $M_S$ vectors and matrices $(Z_{j^{S}}, U_{j^{S}})_{j^{S}=1}^{M_S}$ is private, and not a single output at a given time. This issue is solved by leveraging batching.
    Formally, we can show in this scenario that the sequence $(Z_{j^{S}}, U_{j^{S}})_{j^{S}}$ is $(\varepsilon, \delta_{0} + \delta)$-DP for any $\delta_{0}, \delta\in (0,1)$ and $\varepsilon\in (0,1)$.
    \begin{theorem}\label{thm:privacy.alg}
        For any $\varepsilon\in (0,1)$, $\delta_{0},\delta \in (0,1)$, encoding parameter $m$ and LDP parameter $\varepsilon_{0}>0$, let $p = 2(e^{2\varepsilon_{0}/md(d+3)} + 1)^{-1}$.
        Then if $l^{\star}$, the length of a shuffler batch, satisfies
        $l^{\star}p\geq 14\log(8mT/\delta_{0})$ and:
        {\small\begin{equation}\label{eq:equation_privacy_dp}
            \begin{aligned}
            \sqrt{\left(2 + \left(\frac{\varepsilon l^{\star}}{32d(d+3)\log(8mT/\delta_{0})\sqrt{2T\ln(2T/\delta_{0})}}\right)^{2} \right)^{2} - 4} \geq 1 - 2p +  2\sqrt{\frac{2\log(2mT/\delta_{0})}{l}}&
            \\ + \left(\frac{\varepsilon l^{\star}}{32d(d+3)\log(8mT/\delta_{0})\sqrt{2T\ln(2T/\delta_{0})}}\right)^{2},&\\
            \end{aligned}
        \end{equation}}
        the sequence $(Z_{j^{S}}, U_{j^{S}})_{j^{S}}$ is
        central $(\varepsilon, \delta_{0} + \delta)$-DP.\footnote{We provide the definition of central DP in Def.~\ref{def:centralDP} in App.~\ref{app:proof}. Note that the concept of central DP is at the core for proving JDP results, in fact thanks to Claim $7$ in \citep{Shariff2018contextual} having a sequence $(\wt{V}_{t}, B_{t})_{t}$ is $(\varepsilon, \delta)$-DP implies that a bandit algorithm
        based on this sequence is $(\varepsilon, \delta)$-DP.}
    \end{theorem}

    The result of Thm.~\ref{thm:privacy.alg} is a consequence of the advanced composition theorem~\citep{dwork2010differential}.
    Indeed,  thanks to shuffling, for any batch $j^{S}$, the statistics
    $(Z_{j^S}, U_{j^{S}})$ are $\left( \frac{\sqrt{\varepsilon(1-p)}}{T^{1/4}}, \frac{\delta_{0}\varepsilon}{\sqrt{T}(1-p)}\right)$-DP, since the batch length $l$ is approximately $\frac{\sqrt{T}(1 - p)}{\varepsilon}$.
    As a consequence, when composing them together we get that the central DP level of each batch is  $\wt{\mathcal{O}}\left( \varepsilon\sqrt{\frac{l^{\star}}{T}}\right)$.
    Therefore by advanced composition, since we have a total number of batches $M_S \approx \sqrt{T}$, the total privacy over the sequence of $(Z_{j^{S}}, U_{j^{S}})_{j^{S}}$ is of order $\wt{\mathcal{O}}\left(\varepsilon\sqrt{\frac{l^{\star}}{T}}\times \sqrt{\frac{T}{l^{\star}}} \right)$
    that is to say of order $\wt{\mathcal{O}}\left(\varepsilon \right)$.

    \vspace{.15in}
    \noindent
    \encircle{c} \textbf{Compromised Communication between the bandit algorithm and the users.}
    Similarly to~\citet{Shariff2018contextual}, in the final scenario we consider, an adversary can observe the same data coming from \algo as the users, \ie the stream
    of estimates $(\wt{\theta}_{k^A_t}, \wt{V}_{k^A_t}, \beta_{k^A_t})_{t \in [T]}$.
    Recall that the bandit algorithm uses a dynamic batch schedule based on the determinant technique and it is asynchronous w.r.t.\ the shuffler. This leads to a number of bandit batches roughly of order $\log(T)$. While we have to guarantee privacy on a smaller number of element ($\log(T)$ compared to $\sqrt{T}$ in the shuffler), we are technically limited by the former scenario \encircle{b}. 
    As shown in Prop.~\ref{prop:jdp_guarantee},  \algo is $(\varepsilon, \delta_{0}+ \delta)$-JDP w.r.t.\ the sequence $(\wt{\theta}_{j^A}, \wt{V}_{j^A}, \beta_{j^A})_{j^A}$ since $(Z_{j^{S}}, U_{j^{S}})_{j^{S}}$ is $(\varepsilon, \delta_{0}+ \delta)$-DP.

    \begin{proposition}[JDP guarantee]\label{prop:jdp_guarantee}
        For any $\varepsilon\in (0,1)$, $\varepsilon_{0}>0$, $\delta, \delta_{0}\in (0,1)$, $m\in \mathbb{N}^{\star}$, selecting the length of a shuffler like in Thm.~\ref{thm:privacy.alg} ensures that  the sequence of $(\wt{\theta}_{j^{A}}, \wt{V}_{j^{A}}, \beta_{j^{A}})_{j^{A}}$ is
        $(\varepsilon, \delta + \delta_{0})$-DP. In other words \algo is $(\varepsilon, \delta + \delta_{0})$-JDP.
    \end{proposition}

    Since we are directly leveraging advance composition,  we cannot get any privacy amplification when we consider \encircle{b} and \encircle{c} together.
    Scenario \encircle{c} is indeed the most stringent adversary model in the shuffle-model, limiting the gain in the privacy/regret we can obtain compared to the pure LDP setting. It is however possible to achieve a better privacy/utility trade-off when considering only scenario \encircle{c} (and not \encircle{b}), but we believe it is a much weaker attack scenario. In both scenarios, \encircle{b} and \encircle{c}, the objective is to ensure Joint Differential Privacy. Model \encircle{c} deals with the issue when attackers can submit potentially false contexts to the bandit algorithm and observes the action recommended with the objective to learn the context/reward of a target user. Guaranteeing that this task is difficult is the objective of Joint Differential Privacy. In this paper, we use a deterministic bandit algorithm therefore in terms of privacy scenarios \encircle{b} and \encircle{c} are the same (thanks to the post-processing lemma). However, one could think of using a randomized algorithm and therefore improve the privacy of the whole scheme. 
    
    \begin{remark}
        In online learning, JDP and central-DP are not equivalent definitions. A DP constraint on the actions selected implies that the probability of selecting any action is strictly positive thus hindering the algorithm to select the optimal action. Indeed, as noted in \citep{Shariff2018contextual} (see Claim 13) any central-DP linear contextual bandit algorithm must incur linear regret, whereas in the weaker definition of JDP it is possible to attain a sublinear regret. The fact that the computation of the action is local is necessary to achieve a sublinear regret.
    \end{remark}

\subsection{Regret Analysis of \algo}
In the previous section, we stated several  privacy guarantees of \algo with different attack models. We shall now  show the impact of those privacy guarantees on the regret. As mentioned, shuffling allows to regulate the level and type of privacy desired by trading-off the regret guarantee. In \algo, this trade-off is regulated by the parameter $\varepsilon_0$ which has impact on all the main elements in the privacy and regret analysis (\eg batch size, privacy $p$, etc.).

The first result we provide is a validation of our algorithm. The following proposition shows that \algo recovers the results in~\citep{zheng2021locally}, providing the highest possible \emph{local} DP level at the expense of the regret bound.
\begin{corollary}\label{cor:ldp.regret}
    For any $\varepsilon_{0} > 0$ and $\delta \in (0,1)$ then choosing $\varepsilon = \sqrt{\exp(\varepsilon_{0}) - 1}$ and $\delta_{0} = \delta$ we have that \algo is $\varepsilon_{0}$-LDP and with probability at least $ 1-\delta$ is bounded by:
    {\small\begin{equation}
        \begin{aligned}
            R_{T} \leq \wt{\mathcal{O}}\left( \frac{T^{3/4}\sqrt{e^{\varepsilon_{0}} + 1}}{\sqrt{e^{\varepsilon_{0}} - 1}} + \frac{\log(T)\left(e^{\varepsilon_{0}} + 1\right)^{2}}{4} + \frac{\sqrt{T}}{\sqrt{e^{\varepsilon_{0}} - 1}}\right)
        \end{aligned}
    \end{equation}}
\end{corollary}

On the other hand, Cor.~\ref{cor:regret.fully.optimized} shows that \algo interpolates between the regret of \citep{zheng2021locally} (LDP setting studied under scenario \encircle{a}) and \citep{Shariff2018contextual} (JDP setting studied under scenario \encircle{c}).
The structure of the shuffle-model requires to also consider scenario \encircle{b} that, as mentioned before, poses the highest restriction on the regret bound we can achieve.

\begin{corollary}\label{cor:regret.fully.optimized}
    For any  $\varepsilon \leq \frac{1}{27T^{1/4}}$ and $\delta, \delta_{0}\in (0,1)$, the choices of $\eta=0.5$, $\lambda= \sqrt{T}$, $m=1$ and  $\varepsilon_{0} = \frac{d(d+3)}{2}\ln\left( \frac{2}{1 - \varepsilon^{2/3}T^{1/6}} - 1\right)$ ensures that with
    probability at least $1 - \delta$ the regret of \algo is bounded by:
    \begin{equation}
           R_{T} \leq \frac{4T^{2/3}}{\varepsilon^{1/3}}\Big(S+d+\frac{1}{T^{1/4}}\widetilde{\mathcal{O}}(1)\Big),
            \end{equation}
    where $\widetilde{\mathcal{O}}(\cdot)$ hides poly-log factor (in $T, \delta, \delta_0$) and polynomial factors (in $d$, $L$).     In addition \algo is $(\varepsilon, \delta_{0} + \delta)$-JDP and $6d^2\varepsilon^{2/3}T^{1/6}$-LDP.
\end{corollary}

For the complete regret bound refer to the end of App.~\ref{app:proof}. This shows that the regret bound of \algo is of order $\mathcal{O}\left(\nicefrac{dT^{2/3}}{\varepsilon^{1/3}}\right)$, while being $(\varepsilon, \delta)$-JDP and approximately $(2\varepsilon^{2/3}T^{1/6}, 0)$-LDP.
As expected, this indicates  the regret bound can be improved by sacrificing some level of LDP. However, the $\sqrt{T}$ regret bound of~\citep{Shariff2018contextual} cannot be recovered directly.
While the search for a better upper-bound or a lower-bound is an interesting future direction, we think it would be hard to match such JDP minimax result.
Indeed, shuffling allows to interpolate between JDP (where the best minimax bound is $\sqrt{T}$) and LDP (where the best known upper bound is $T^{3/4}$). Since we will always have a non-zero LDP level of privacy in the considered ESA shuffle model, we believe it is almost impossible to achieve $\sqrt{T}$ regret in particular.\footnote{Note that in multi-armed bandit (MAB), it is possible to achieve a minimax regret bound of order $\sqrt{T}$ both in central DP and LDP \citep{ren2020multi,Basu2019differential}. We think this is an important aspect leveraged by~\citet{tenenbaum2021shuffle} for shuffling in MAB. In addition, as already mentioned, they considered a weaker attack model.}

\subsubsection{Proof Sketch}
The proof of this theorem is presented in details in App.~\ref{app:proof}.
To understand this result however we present how we build the confidence intervals around the parameter $\theta^{\star}$.
As noticed in \citep{Shariff2018contextual}, the estimator $\wt{\theta}_{j}$ is the result of a ridge regression computed
by a design matrix regularized by a regularizer which is a function of the time. Therefore in order to apply Prop.~$4$ in \citep{Shariff2018contextual}
we need to ensure that our estimator $\wt{V}_{j}$ of the design matrix, $\sum_{t} x_{t,a_{t}}x_{t,a_{t}}^{\intercal}$, is unbiased and to bound with high probability
the deviation with respect to the design matrix. We also need the same type of guarantees with respect to the vector $\wt{B}_{j}$ and $\sum_{t} r_{t}x_{t,a_{t}}$.

\paragraph{Computation of our Estimators.}

The bandit algorithm receives the estimate $(Z_{j^{S}}, U_{j^{S}})$ from the shuffler but given the data those estimates are biased.
For a couple of vector and reward, $x$ and $r$, let us note $\mathcal{M}_{\text{LDP}}(x,r) = (b,w)$, so that
\begin{align*}
    &\mathbb{E}\left( b_{k,q}\mid x,r\right) = \frac{p}{2} + (1 - p)\left[ \mathds{1}_{\{q < \lceil rx_{k}m\rceil\}} + \mathds{1}_{\{q = \lceil rx_{k}m\rceil\}}(mrx_{k} - \lceil rx_{k}m\rceil +1)\right]\\
    &\mathbb{E}\left( w_{k,l}\mid x,r\right) = \frac{p}{2} + (1 - p)\left[ \mathds{1}_{\{q < \lceil x_{l}x_{k}m\rceil\}} + \mathds{1}_{\{q = \lceil x_{l}x_{k}m\rceil\}}(mx_{l}x_{k} - \lceil x_{l}x_{k}m\rceil +1)\right]
\end{align*}
for all $k,l\leq d$ and $q\leq m$. Therefore, we introduce  a debiased estimator for computing the estimators of \algo,  written as follows:\footnote{Note that this is an alternative but equivalent form to the one used in Alg.~\ref{alg:complete.algo}.}
\begin{equation}
    \begin{aligned}
        \wt{V}_{j^{A}} = \sum_{t=1}^{t_{j^{A}}} \frac{x_{t,a_{t}}x_{t,a_{t}}^{\top}}{2L^{2}} + H_{j^{A}} + \lambda_{j^{A}}I_{d} ~~~ \text{ and } ~~~
        \wt{B}_{j^{A}} = \sum_{l=1}^{t_{j^{A}}} \frac{r_{l}x_{l,a_{l}}}{2L} + h_{j^{A}},
    \end{aligned}
\end{equation}
where, for all batches, $H_{j^{A}} + \lambda^{j^{A}}I_{d}$ is with high probability a symmetric positive definite matrix decomposed as the sum of zero mean noise and a regularization $\lambda_{j^{A}}$, and $ h_{j^{A}}$
is a vector of zero mean noise. Both noises are due to the noise introduced in by the local randomizer $\mathcal{M}_{\text{LDP}}$. 
In addition, as we show in App.~\ref{app:proof} controlling the eigenvalues of the regularizer $H_{j^{A}} + \lambda^{j^{A}}I_{d}$ and the noise $h_{j^{A}}$
is bounded roughly by $\sqrt{t_{j^{A}}}$. Therefore thanks to Prop.~$4$ in \citep{Shariff2018contextual},  the following proposition holds.

\begin{lemma}[Confidence Ellipsoid]\label{lem:confidence_ellipsoid}
    For any $\delta\in (0,1)$, $\varepsilon_{0}>0$, $p = \frac{2}{e^{2\varepsilon_{0}/(md(d+3))} + 1}$ and $\lambda >0$, we have with probability at least $1 - \delta$ that:
    \begin{equation}\label{eq:confidence_width}
        \begin{aligned}
            \forall j^{A}\leq M_{S}, \qquad \| \theta^{\star} - \wt{\theta}_{j^{A}} \|_{\wt{V}_{j^{A}}^{-1}} \leq \beta_{j^{A}} := \sigma \sqrt{8\log\left(\frac{2t_{j^{A}}}{\delta}\right) + d\log\left(3 + \frac{t_{j^{A}}L^{2}}{\lambda_{j^{A}}}\right)} + S\sqrt{3\lambda_{j^{A}}}&\\
            + \frac{d}{\sqrt{\lambda_{j^{A}}}}\left(2\sqrt{p\left(1 - \frac{p}{2}\right)t_{j^{A}}m\log\left(\frac{2t_{j^{A}}}{\delta}\right)} + \frac{8\log(2t_{j^{A}}/\delta)}{3} +  \frac{\sqrt{8}}{m}\sqrt{t_{j^{A}}\log\left(\frac{2t_{j^{A}}}{\delta}\right)}\right)&
        \end{aligned}
    \end{equation}
    where $M_{S} = \nicefrac{T}{l^{\star}}$ is the number of shuffler batch and for all $j^{A}\leq M_{S}$,
    \begin{equation}\label{eq:lambda_reg}
        \lambda_{j^{A}} = \frac{\sqrt{8t_{j^{A}}\ln(2t_{j^{A}}/\delta)}}{m} + \frac{2\sqrt{8t_{j^{A}}\ln(2t_{j^{A}}/\delta)}}{(1 - p)\sqrt{m}} + \lambda
    \end{equation}
\end{lemma}

Given the definition of the confidence ellipsoid above, we can analyze the regret using a standard regret analysis for algorithms
using the optimism-in-the-face-of-uncertainty principle. For a generic set of privacy parameters $\varepsilon_0$, $\varepsilon$ and $\delta_0$, the regret bound of \algo is given in the following theorem.

\begin{theorem}\label{thm:regret.generic}
    For any $\delta, \delta_{0}\in (0,1)$, $\varepsilon, \varepsilon_{0}\in (0,1)$ and $T\geq 1$, let $p=2(e^{2\varepsilon_{0}/md(d+3)} +1)^{-1}$then with probability at least $1 - \delta$, the regret of
    Alg.~\ref{alg:esa_alg} is bounded by:
    \begin{itemize}
        \item If $p^{2}(1 - p) \leq \frac{7T^{-1/2}  \varepsilon}{64\sqrt{2\ln(2T/\delta_{0})}d(d+1)} $:
        {\small\begin{equation}
        \begin{aligned}
            &R_{T} \leq \frac{2\sqrt{3}(S + md)T^{3/4}}{\sqrt{1-p}}\sqrt{\left(1 + \eta\right)\log\left(1 + \frac{T}{d\lambda}\right)} \\
            &+ \frac{dLm}{\sqrt{\lambda}}\left(1 + \frac{d^{3/2}\log\left(\frac{L^{2}T}{d} + \frac{16\sqrt{T}\log\left(2T/\delta\right)}{(1 - p)}\right)^{3/2}}{\log(1 + \eta)} \right)\frac{14\log(8mT/\delta_{0})}{p^{2}}
        \end{aligned}
        \end{equation}}
        \item If $p^{2}(1 - p) \geq \frac{7T^{-1/2}  \varepsilon}{64\sqrt{2\ln(2T/\delta_{0})}d(d+1)} $:
        {\small\begin{equation}\begin{aligned}
        &R_{T}\leq \frac{2\sqrt{3}(S + md)T^{3/4}}{\sqrt{1-p}}\sqrt{\left(1 + \eta\right)\log\left(1 + \frac{T}{d\lambda}\right)} \\
            &+
            \frac{264}{\sqrt{\lambda}}\sqrt{2}d^{3}\log\left(\frac{8mT}{\delta_{0}}\right)^{3/2}Lm\left(1 + \frac{d^{3/2}\log\left(\frac{L^{2}T}{d} + \frac{16\sqrt{T}\log\left(2T/\delta\right)}{(1 - p)}\right)^{3/2}}{\log(1 + \eta)} \right)\frac{\sqrt{T}(1-p)}{\varepsilon}
        \end{aligned}\end{equation}}
      \end{itemize}
\end{theorem}

The first term of the regret in Thm.~\ref{thm:regret.generic} highlights the regret coming from the local privacy guarantees whereas the second term
is coming from the mismatch between the batch of the shuffler and the batch of the bandit algorithm. Indeed, when the bandit algorithm
updates its batch it means that during the last shuffler batch the determinant condition was satisfied at some point during the shuffler batch.
However, the impact on the regret during this shuffler batch can only be bounded by the length of a shuffler batch times the maximum reward possible.
But given Thm.~\ref{thm:privacy.alg} the length of a shuffler batch scales with $\wt{\mathcal{O}}\left( \nicefrac{\sqrt{T}}{\varepsilon}\right)$. Hence the final regret scales
with $\wt{\mathcal{O}}\left( \nicefrac{T^{3/4}}{\sqrt{1 - p}} + \nicefrac{\sqrt{T}}{\varepsilon}\right)$.
As a consequence, Cor.~\ref{cor:ldp.regret} and Cor.~\ref{cor:regret.fully.optimized} are obtained by optimizing for the highest privacy level and smaller regret bound, respectively.

\begin{remark}\label{rem:weaker.privacy.model}
    A better regret bound can be obtained in the setting of~\citep{tenenbaum2021shuffle}, where the adversary only observes the output of the shuffler for one time step. In particular, this allows to improve the privacy analysis and obtain a generic regret bound of order $\mathcal{O}\left( \nicefrac{T^{3/4}}{\sqrt{1 - p}} + \nicefrac{\log(T)}{\varepsilon^{2}}\right)$ that once optimized leads to a regret bound of $T^{3/5}/\varepsilon^{2/5}$ which is much closer to the best JDP regret bound. However, we think this setting is less practical than the one considered in this paper.
\end{remark}

\section{Conclusion}\label{sec:conclusion}
We introduced \algo, an algorithm for linear contextual bandits that achieves a trade-off between joint and local differential privacy. Our algorithm is a variant of batched \linucb with dynamic schedule using a variant of the binary sum method to achieve privacy. Thanks to an asynchronous batch schedule between shuffler and bandit algorithm, it is able to take advantage of the privacy amplification through shuffling to reduce the gap between JDP and LDP regret bound. 

An interesting question raised by our paper is whether it is possible to use a synchronous schedule between the shuffler and the bandit algorithm, \eg by making the shuffler batch data dependent. We believe this would require to use some private technique~\citep[\eg sparse vector technique by][]{DworkNRRV09} to guarantee privacy at the output of the shuffler. Another direction inspired by our paper is to gain a better understanding about the intrinsic limitations of differential privacy in linear contextual bandits by studying lower-bounds for these settings. 

\newpage{}
\section*{Acknowledgments and Disclosure of Funding.}
V. Perchet acknowledges support from the French National Research Agency (ANR) under grant
number \#ANR-$19$-CE$23$-$0026$ as well as the support grant, as well as from the grant “Investissements
d’Avenir” (LabEx Ecodec/ANR-$11$-LABX-$0047$).

\bibliographystyle{plainnat}
\bibliography{rl_refs, privacy_refs}

\clearpage
\begin{appendix}
\addcontentsline{toc}{section}{Appendix} 
\part{Appendix} 
\parttoc 
\section{Local Privatizer $\mathcal{M}_{\text{LDP}}$}\label{app:ldp_mechanism}

In this appendix, we present the privacy-preserving mechanism $\mathcal{M}_{\text{LDP}}$ used in this paper.

\begin{algorithm}[htp]
    \caption{Local Privatizer $\mathcal{M}_{\text{LDP}}$}
    \label{alg:esa_alg}
    \SetAlgoLined
    \DontPrintSemicolon
    \KwIn{context: $x\in\mathbb{R}^{d}$, reward: $r\in[0,1]$, context bound: $L$, privacy parameter: $\varepsilon_{0}$, encoding parameter: $m$}
    {\color{blue}\textbf{//*Encoder*//}}\;
    Set $\tilde{y} = \frac{rx}{2L} + \frac{1}{2}$ and $\tilde{z} = \frac{xx^{\intercal}}{2L^{2}} + \frac{\mathds{1}\mathds{1}^{\intercal}}{2}$\;
    \For{$j = 1, \dots, d$}{
      Compute $\mu_{j} = \left \lceil \tilde{y}_{j}\cdot m \right \rceil$ and $p_{j} = m\cdot\tilde{y}_{j} - \mu_{j} + 1$\;
      \For{$k=1, \dots, m$}{
        Let $b_{j,k} = \left\{\begin{matrix}
          1 & \text{ if } k < \mu_{j}\\
          \text{Ber}(p_{j}) & \text{ if } k = \mu_{j}\\
          0 & \text{ if } k > \mu_{j}\\
         \end{matrix}\right.$
      }
    }
    \For{$i = 1, \dots, d$}{
        \For{$j=1, \dots, i$}{
          Compute $\kappa_{i,j} =  \left \lceil \tilde{z}_{i,j}\cdot m \right \rceil$ and $q_{i,j} = m\cdot\tilde{z}_{i,j} - \kappa_{i,j} + 1$\;
          \For{$k=1, \dots, m$}{
            Let $w_{i,j,k} = \left\{\begin{matrix}
              1 & \text{ if } k < \kappa_{i,j}\\
              \text{Ber}(q_{i,j}) & \text{ if } k = \kappa_{i,j}\\
              0 & \text{ if } k > \kappa_{i,j}\\
            \end{matrix}\right.$\;
            Let $w_{j,i, k} = w_{i,j,k}$\;
          }
        }
      }
      {\color{blue}\textbf{//*Local Randomizer*//}}\;
      Set probabilities $p = \frac{2}{\exp(\nicefrac{2\varepsilon_{0}}{md(d+3)}) + 1}$ and compute
      private values $\tilde{b}_{j} = \left( R_{p}^{0/1}(b_{j,1}), \dots, R_{p}^{0/1}(b_{j,m})\right)$ for all $j\in\llbracket 1,d\rrbracket$,
      $\tilde{w}_{i,j} = \left( R_{p}^{0/1}(w_{i,j,1}), \dots, R_{p}^{0/1}(w_{i,j,m})\right)$ for all $i\in \llbracket 1,d\rrbracket$, $j\leq i$ \;

\end{algorithm}

\begin{algorithm}[htp]
    \caption{Local Randomizer $R^{0/1}_{p}$}
    \label{alg:local_randomizer}
    \SetAlgoLined
    \KwIn{probability: $p$, $x\in\{0, 1\}$}
    Let $\mathbf{b} \sim \text{Ber}(p)$\;
    \eIf{$\mathbf{b} = 0$}{Return $x$}{Return $\text{Ber}(1/2)$}
\end{algorithm}

\section{Proofs}\label{app:proof}

In this appendix, we provide the full derivation of the results stated in the main text.
We start introducing the notion of central $(\varepsilon, \delta)$-DP that is widely used in the proofs.

\begin{definition}\label{def:centralDP}
A randomized mechanism, $\mathcal{M}: \mathbb{R}^{d} \rightarrow \mathcal{Z}$, is said to be central $(\varepsilon, \delta)$ differential private (DP) if for all sequence of values $z\in \mathcal{R}^{d}$ and
        $z'$ such that there exists a unique $i\leq t$ for which $z_{i}\neq z_{i}'$ and for all $j\neq i$, $z_{j} = z_{j}'$ then
        $$\mathbb{P}\left( \mathcal{M}(z)\in A\mid z\right)\leq e^{\varepsilon}\mathbb{P}\left( \mathcal{M}(z)\in A\mid z'\right) + \delta $$ for any $A\subset \text{Range}(\mathcal{M})$.
\end{definition}
Note that the concept of central DP is at the core for proving JDP results, in fact thanks to Claim $7$ in \citep{Shariff2018contextual} having a sequence $(\wt{V}_{t}, B_{t})_{t}$ is $(\varepsilon, \delta)$-DP implies that a bandit algorithm based on this sequence is $(\varepsilon, \delta)$-DP.

\subsection{Proof of Lem.~\ref{lem:confidence_ellipsoid}}

    Here, we detail how to obtain the confidence intervals around $\theta^{\star}$ using the privatized estimator $\wt{\theta}_{j}$ for any batch $j^{S}\leq M_{S}$ ( with $M_{S} = Tl^{\star}$ the total number
    of batches from the shuffler side).
    First, let's define the sequence of random variables $(Y_{t,k,l,q})_{t\leq T, k,l\leq d, q\leq m}$, $(Z_{t,k,l,q})_{t\leq T, k,l\leq d, q\leq m}$ two independent
    sequences of i.i.d.\ Bernoulli distributed random variable with parameters $p = 2/(\exp(2\varepsilon_{0}/md(d+3)) + 1)$ and $1/2$ and such that for all $k,l\leq d$, $Y_{t,k,l,q} = Y_{t,l,k,q}$ and $Z_{t,k,l,q} = Z_{t,l,k,q}$.
    For every $(t,k,l,q)\in [T]\times [d]\times[d]\times [m]$, $Y_{t,k,l,q}$ is sampled by Alg.~\ref{alg:local_randomizer} if $Y_{t,k,l,q} = 1$ then it return the random variable
    $Z_{t,k,l,q}$ otherwise it returns the true data.

    In addition, let's define
    $(A_{t,k,l} = w_{t,k,l,\kappa_{t,k,l}})_{t\leq T, k,l\leq d}$ a sequence of Bernoulli random variable with parameter $(q_{t,k,l})_{t\leq T, k,l\leq d}$ defined by the two sequences $(\wt{z}_{t,k,l})_{t\leq T, k,l\leq d}$ and $(\kappa_{t,k,l})_{t\leq T, k,l\leq d}$
    in the mechanism $\mathcal{M}_{\text{LDP}}$, Alg.~\ref{alg:esa_alg}. Finally, let's note $ $ the sequence of data computing by the encoding part of Alg.~\ref{alg:esa_alg}.

    For any batch $j^{A}\leq M_{S}$, we can write the approximate design matrix and vector $B_{j}$ as follows for every coordinate $k,l\leq d$:
    \begin{equation}
        \begin{aligned}
            \wt{V}_{j,k,l} = &\frac{1}{m(1 - p)}\sum_{t=1}^{t_{j}}\sum_{q=1}^{m} Y_{t,k,l,q}Z_{t,k,l, q} - \frac{p}{2} + \sum_{t=1}^{t_{j}} \frac{x_{t,a_{t}}x_{t,a_{t}}^{\intercal}}{2L^{2}} + + 2\lambda_{j}\mathds{1}_{\{k = l\}}\\
            &+ \frac{1}{m}\sum_{t=1}^{t_{j}} A_{t,k,l} - (m\wt{z}_{t, k, l} - \kappa_{t, k, l} + 1) + \frac{1}{m(1-p)}\sum_{t=1}^{t_{j}} \sum_{q=1}^{m} (p - Y_{t,k,l, q})w_{t,k,l, q}
        \end{aligned}
    \end{equation}
    where $\lambda_{j}$ is defined in Eq.~\eqref{eq:lambda_reg}.
    \begin{equation}
        \begin{aligned}
            \wt{B}_{j,k} = \frac{1}{m(1 - p)}\sum_{l=1}^{t_{j}}\sum_{q=1}^{m} \left( \tilde{b}_{l,i,q} - \frac{p}{2}\right) - \frac{t_{j}}{2}
        \end{aligned}
    \end{equation}

    Now, given an well-chosen regularization $\lambda_{j}$ the approximate design matrix $\wt{V}_{j}$ can be written as the sum of the true design matrix $\sum_{t} x_{t,a_{t}}x_{t,a_{t}}^{\intercal}$ and
    a time-varying regularizer similar to \citep{Shariff2018contextual}. We just need to bound with high probability the deviation of the eigenvalues of $ \wt{V}_{j} - \sum_{t=1}^{t_{j}} \frac{x_{t,a_{t}}x_{t,a_{t}}^{\intercal}}{2L^{2}} - \lambda_{j}I_{d}$.

    Let's consider a vector $v\in \mathbb{R}^{d}$ such that $\| v\|_{2} = 1$ then for any time $t_{j}\leq T$ and $\delta\in (0,1)$ we have with probability at least $1 - \delta$:
    \begin{equation}
        \begin{aligned}
            \left| \left\langle v, \left(\sum_{t=1}^{t_{j}}\sum_{q=1}^{m} Y_{t,.,.,q}Z_{t,.,., q} - \frac{p\mathds{1}\mathds{1}^{\intercal}}{2}\right)v\right\rangle\right| \leq 2\sqrt{2t_{j}m\ln(2/\delta)}
        \end{aligned}
    \end{equation}
    Therefore because the matrix $\left(\sum_{t=1}^{t_{j}}\sum_{q=1}^{m} Y_{t,.,.,q}Z_{t,.,., q} - \frac{p\mathds{1}\mathds{1}^{\intercal}}{2}\right)$ is symmetric we have that with high probability:
    {\small
    $$\max\left\{ \left|\lambda_{\min}\left(\sum_{t,q} Y_{t,.,.,q}Z_{t,.,., q} - \frac{p\mathds{1}\mathds{1}^{\intercal}}{2}\right) \right|, \lambda_{\max}\left(\sum_{t,q} Y_{t,.,.,q}Z_{t,.,., q} - \frac{p\mathds{1}\mathds{1}^{\intercal}}{2}\right)\right\} \leq 2\sqrt{2t_{j}m\ln(2/\delta)}$$
    }
    where $\lambda_{\min}$ and $\lambda_{\max}$ are the minimum and maximum eigenvalues. Similarly, using the martingale difference structure,we have that for
    any $v\in \mathbb{R}^{d}$, $\| v\|_{2}\leq 1$ and $\delta\in (0,1)$, we have with probability at least $1 - \delta$:
    \begin{equation}
        \begin{aligned}
            \left| \left\langle v, \left(\sum_{t=1}^{t_{j+1}} \sum_{q=1}^{m} (p\mathds{1}\mathds{1}^{\intercal} - Y_{t,.,., q})w_{t,.,., q}\right)v\right\rangle\right| \leq 2\sqrt{2t_{j+1}m\ln(2/\delta)}
        \end{aligned}
    \end{equation}
    and
    \begin{equation}
        \begin{aligned}
            \left| \left\langle v, \left(\sum_{t=1}^{t_{j+1}} A_{t} - (m\wt{z}_{t} - \kappa_{t} + 1)\right)v\right\rangle\right| \leq 2\sqrt{2t_{j+1}\ln(2/\delta)}
        \end{aligned}
    \end{equation}

    Indeed, for every $t\leq T$, let's define the filtration $\mathcal{F}_{t}$ which is the filtration generated by all the history up to time $t$ included except for the noise added by the mechanism $\mathcal{M}_{\text{LDP}}$
    that is to say $\mathcal{F}_{t} = \sigma((x_{l,a_{l}}, r_{l})_{l\leq t}, (Y_{l,i,j,q})_{l<t-1, i,j\leq d, q\leq m}, (Z_{l,i,j,q})_{l<t-1, i,j\leq d, q\leq m}, (w_{t,k,l,q})_{l<t-1, i,j\leq d, q\leq m})$. Therefore, we have that:
    \begin{equation}
        \begin{aligned}
            &\mathbb{E}\left((p - Y_{t,k,l,q})w_{t,k,l,q} \mid \mathcal{F}_{t}\right) = \mathbb{E}(p - Y_{t,k,l,q})\mathbb{E}\left(w_{t,k,l,q} \mid \mathcal{F}_{t}\right) = 0\\
            &\mathbb{E}\left(A_{t,k,l} - (m\wt{z}_{t,k,l} - \kappa_{t,k,l} + 1)\mid \mathcal{F}_{t}\right) = \mathbb{E}(A_{t,k,l}\mathcal{F}_{t}) - (m\wt{z}_{t,k,l} - \kappa_{t,k,l} + 1) = 0
        \end{aligned}
    \end{equation}
    because $Y_{t}$ is independent of $\mathcal{F}_{t}$ and $w_{t}$. The second equality comes from the fact that given $\mathcal{F}_{t}$, $A_{t,k,l}$ is a Bernoulli random variable
    with parameter $m\wt{z}_{t,k,l} - \kappa_{t,k,l} + 1$.

    Hence, when choosing $\lambda_{j} = \frac{\sqrt{8t_{j}\ln(2t_{j}/\delta)}}{m} + \frac{2\sqrt{8t_{j}\ln(2t_{j}/\delta)}}{(1 - p)\sqrt{m}}$,
    we have that with probability at least $1 - \delta$:
    {\small
    \begin{align}
        \forall j\leq M_{S}, \qquad &\lambda_{\min}\left(\wt{V}_{j} -\sum_{t=1}^{t_{j}} \frac{x_{t,a_{t}}x_{t,a_{t}}^{\intercal}}{2L^{2}}  \right) \geq \frac{\sqrt{8t_{j}\ln\left(\frac{2t_{j}}{\delta}\right)}}{m} + \frac{2\sqrt{8t_{j}\ln\left(\frac{2t_{j}}{\delta}\right)}}{(1 - p)\sqrt{m}} +\\
        &\lambda_{\max}\left(\wt{V}_{j} -\sum_{t=1}^{t_{j}} \frac{x_{t,a_{t}}x_{t,a_{t}}^{\intercal}}{2L^{2}}  \right) \leq \frac{2\sqrt{8t_{j}\ln\left(\frac{2t_{j}}{\delta}\right)}}{m} + \frac{4\sqrt{8t_{j}\ln\left(\frac{2t_{j}}{\delta}\right)}}{(1 - p)\sqrt{m}}
    \end{align}}

    In addition, with the same reasoning, we have with probability at least $1 - \delta$:
    {\small\begin{equation}
        \begin{aligned}
            \left\| \sum_{l=1}^{t_{j+1}} \frac{r_{l}x_{l,a_{l}}}{2L} -  B_{j}\right\| \leq 2\sqrt{dp\left(1 - \frac{p}{2}\right)t_{j}m\log\left(\frac{2t_{j}}{\delta}\right)} + \frac{4}{3}\sqrt{d}\log\left(\frac{2t_{j}}{\delta}\right) +  \frac{2}{m}\sqrt{dt_{j}\log\left(\frac{2t_{j}}{\delta}\right)}
        \end{aligned}
    \end{equation}}

    Therefore, using Prop. $5$ in \citep{Shariff2018contextual}, we have that the result.

    \subsection{Proof of Prop.~\ref{prop:ldp_guarantee}}

    We now move to prove the following proposition which implies Prop.~\ref{prop:ldp_guarantee};

    \begin{proposition}\label{prop:ldp_mechanism}
        For any encoding parameter $m\in \mathbb{N}^{\star}$ and
        LDP parameter $\varepsilon_{0}>0$, $\mathcal{M}_{\text{LDP}}(x,r)$ is $\varepsilon_{0}$-LDP for any $\| x\| \leq L$ and $r\in [0,1]$.
    \end{proposition}
    \begin{proof}
        For any $x,x'\in \mathbb{R}^{d}$ and $r,r'\in [0,1]$ such that $\|x\| \leq L$ and $\|x'\|\leq L$ let's note
        $\mathcal{M}_{\text{LDP}}(x,r) = \left( (\tilde{w}_{i,j})_{i,j\leq d}, (\tilde{b}_{j})_{j\leq d}\right)\in \{0,1\}^{d^{2}m \times dm}$ and $\mathcal{M}_{\text{LDP}}(x',r') = \left( (\tilde{w}_{i,j}')_{i,j\leq d}, (\tilde{b}_{j}')_{j\leq d}\right)\in \{0,1\}^{d^{2}m \times dm}$. Therefore, let's consider
        a tuple $(W_{0}, B_{0}) \in \{0,1\}^{d^{2}m \times dm}$ then we want to show that:
        \begin{align}
            \mathbb{P}\left( \mathcal{M}_{\text{LDP}}(x,r) = (W_{0},B_{0})\right) \leq e^{\varepsilon_{0}}\mathbb{P}\left( \mathcal{M}_{\text{LDP}}(x',r') = (W_{0},B_{0})\right)
        \end{align}
        But we have:
        \begin{equation}
            \begin{aligned}
                &\mathbb{P}\left(\forall i,j\leq d, \tilde{w}_{i,j} = W_{0,i,j}, \tilde{b}_{j} = B_{0,j}\right) = \mathbb{P}\left(\forall i,j\leq d, \tilde{w}_{i,j} = W_{0,i,j}\right) \mathbb{P}\left(\forall j\leq d, \tilde{b}_{j} = B_{0,j}\right) \\
            \end{aligned}
        \end{equation}
        In addition, because the mechanism $R_{p}^{0/1}$ is an example of a randomized response mechanism \citep{dwork2010differential}, we have that for all $j\leq d, q\leq m$,
        $\mathbb{P}( R_{p}^{0/1}(b_{j,m})\mid b_{j,m}) \leq (2/p - 1)\mathbb{P}( R_{p}^{0/1}(b_{j,m}')\mid b_{j,m}')$. Therefore, because of the independence of the sequence
        $(\tilde{b}_{j})_{j}$:
        \begin{equation}
            \begin{aligned}
                \mathbb{P}\left(\forall j\leq d, \tilde{b}_{j} = B_{0,j}\right) &= \prod_{j,q} \mathbb{P}\left(\tilde{b}_{j,q} = B_{0,j,q}\right) \\
                &\leq \prod_{j,q} \mathbb{P}\left(\tilde{b}_{j,q}' = B_{0,j,q}\right)\left(\frac{2}{p} - 1\right) = \left(\frac{2}{p}-1\right)^{dm}\mathbb{P}\left(\forall j, \tilde{b}_{j}' = B_{0,j}\right) \\
            \end{aligned}
        \end{equation}
        For all $i,j\leq d$, we have that $\tilde{w}_{i,j} = \tilde{w}_{j,i}$ therefore:
        \begin{equation}
            \begin{aligned}
                \mathbb{P}\left(\forall i,j\leq d, \tilde{w}_{i,j} = W_{0,i,j}\right) &= \prod_{i, j\leq i, q} \mathbb{P}\left(\tilde{w}_{i,j,q} = W_{0,i,j,q} \right) \\
                &\leq  \prod_{i, j\leq i, q} \left(\frac{2}{p} - 1\right)\mathbb{P}\left(\tilde{w}_{i,j,q}' = W_{0,i,j,q} \right) \\
                &= \left(\frac{2}{p} - 1\right)^{md(d+1)/2}\mathbb{P}\left(\forall i,j\leq d, \tilde{w}_{i,j}' = W_{0,i,j}\right)
            \end{aligned}
        \end{equation}
        Hence the resulting when setting $p = \frac{2}{\exp\left(\frac{\varepsilon_{0}}{md(d+3)/2}\right) +1}$.
    \end{proof}

      \subsection{Proof of Thm.~\ref{thm:privacy.alg}}

      Before proving the JDP guarantees of our algorithm, that is to say Thm.~\ref{thm:privacy.alg}. We first prove the following proposition that is a consequence of
      Thm.~$5.4$ in \citep{cheu2019shuffling}.
      \begin{proposition}\label{prop:privacy_per_batch}
        For any $\delta_{0}, \delta\in (0,1)$, number of batch $M_{S}$ and length $l$, encoding parameter $m$, LDP parameter $0 < \varepsilon_{0} \leq \ln\left( \frac{l}{(7\ln(8m/\delta_{0}))} - 1)\right)$ and
        for all batch $j\leq M_{S}$ of length $l$, the statistics $\left(Z_{j}, U_{j}\right)$ computed by the shuffler
        (with $p = 2/(e^{2\varepsilon_{0}/md(d+3)} + 1))$) are $(\varepsilon_{j,c}, \delta + \delta_{0})$-DP with
        {\small\begin{equation}
           \frac{\varepsilon_{j,c}}{2d(d+3)\sqrt{8m\log(8m/\delta_{0})}} = \left(1 - \left(p - \sqrt{\frac{2p\log\left(\frac{2m}{\delta_{0}}\right)}{l}} \right) \right) \sqrt{\frac{32\log(8m/\delta_{0})}{l\left(p - \sqrt{\frac{2p\log(8\delta_{0}/m)}{l}} \right)}}
        \end{equation}}
      \end{proposition}
      \begin{proof}[of Prop.~\ref{prop:privacy_per_batch}]
       Let's consider $\delta\in (0,1)$ and define
        {\small\begin{align*}
            E_{\delta} = \bigcap_{T=1}^{+\infty}\Bigg\{&\left\|\frac{1}{m(1 - p)}\sum_{t=1}^{T}\sum_{q=1}^{m} Y_{t,.,.,q}Z_{t,.,., q} - \frac{p}{2}\mathds{1}\mathds{1}^{\intercal}\right\| \\
            &+ \left\|\frac{1}{m}\sum_{t=1}^{T} A_{t} - (m\wt{z}_{t} - \wt{\theta}_{t} + 1)\right\|\\
            &+ \left\|\frac{1}{m(1-p)}\sum_{t=1}^{T} \sum_{q=1}^{m} (p - Y_{t,.,., q})w_{t,.,., q} \right\| \leq \frac{\sqrt{8T\ln(2T/\delta)}}{m} + \frac{2\sqrt{8T\ln(2T/\delta)}}{(1 - p)\sqrt{m}}\Bigg\}
        \end{align*}}
        This event is such that $\mathbb{P}(E_{\delta}) \geq 1 - \delta$. Therefore for a batch $j$ and any event $A$, we have that:
        {\small\begin{align*}
            &\mathbb{P}\left( \left( \mathcal{M}_{LDP}(x_{\sigma_{j}(t)}x_{\sigma_{j}(t)}^{\intercal}, r_{\sigma_{j}(t)}x_{\sigma_{j}(t)}))\right)_{t\in \llbracket t_{j}+1, t_{j+1}\rrbracket} \in A \right) = \\
            &\mathbb{P}\left( \left( \mathcal{M}_{LDP}(x_{\sigma_{j}(t)}x_{\sigma_{j}(t)}^{\intercal}, r_{\sigma_{j}(t)}x_{\sigma_{j}(t)}))\right)_{t\in \llbracket t_{j}+1, t_{j+1}\rrbracket} \in A, \mathcal{E}_{\delta} \right)\\
            & + \mathbb{P}\left( \left( \mathcal{M}_{LDP}(x_{\sigma_{j}(t)}x_{\sigma_{j}(t)}^{\intercal}, r_{\sigma_{j}(t)}x_{\sigma_{j}(t)}))\right)_{t\in \llbracket t_{j}+1, t_{j+1}\rrbracket} \in A, \mathcal{E}_{\delta}^{c} \right)
        \end{align*}}
        Therefore, we have that:
        {\small\begin{align*}
        &\mathbb{P}\left( \left( \mathcal{M}_{LDP}(x_{\sigma_{j}(t)}x_{\sigma_{j}(t)}^{\intercal}, r_{\sigma_{j}(t)}x_{\sigma_{j}(t)}))\right)_{t\in \llbracket t_{j}+1, t_{j+1}\rrbracket} \in A \right) \leq  \\
        &\mathbb{P}\left( \left( \mathcal{M}_{LDP}(x_{\sigma_{j}(t)}x_{\sigma_{j}(t)}^{\intercal}, r_{\sigma_{j}(t)}x_{\sigma_{j}(t)}))\right)_{t\in \llbracket t_{j}+1, t_{j+1}\rrbracket} \in A, \mathcal{E}_{\delta} \right)
        + \delta
        \end{align*}}
        And thanks to the definition of privacy with shuffling we have that:
        \begin{align*}
        &\mathbb{P}\left( \left( \mathcal{M}_{LDP}(x_{\sigma_{j}(t)}x_{\sigma_{j}(t)}^{\intercal}, r_{\sigma_{j}(t)}x_{\sigma_{j}(t)}))\right)_{t\in \llbracket t_{j}+1, t_{j+1}\rrbracket} \in A \right) \leq \\
        &\mathbb{P}\left( \left( \mathcal{M}_{LDP}(x_{\sigma_{j}(t)}'(x_{\sigma_{j}(t)}')^{\intercal}, r_{\sigma_{j}(t)}x_{\sigma_{j}(t)}))\right)_{t\in \llbracket t_{j}+1, t_{j+1}\rrbracket} \in A \right)\exp(\varepsilon_{j,c}) + \delta_{0}
        + \delta
        \end{align*}
        where $\varepsilon_{j,c}$ is such that:
        \begin{equation}
            \begin{aligned}
                \frac{\varepsilon_{j,c}}{2d(d+3)\sqrt{8m\log\left(\frac{8m}{\delta_{0}}\right)}} = \left(1 - \left(p - \sqrt{\frac{2p\log\left(\frac{2m}{\delta_{0}}\right)}{l}} \right) \right) \sqrt{\frac{32\log(8m/\delta_{0})}{l\left(p - \sqrt{\frac{2p\log(8\delta_{0}/m)}{l}} \right)}}
            \end{aligned}
        \end{equation}
        according to Thm.~$5.4$ in \citep{cheu2019shuffling}.
      \end{proof}

      Now let's consider a set of parameters $\delta_{0}, \delta\in (0,1)$ and $\varepsilon,\varepsilon_{0}\in(0,1)$ and a length $l$ that satisfies Eq.~\eqref{eq:equation_privacy_dp}.
      Such length $l$ exists for any $p\in [0,1]$ as
    {\small\begin{equation}
        \begin{aligned}
            \text{lim}_{l\rightarrow +\infty} &2\sqrt{\frac{2\log(2m/\delta_{0})}{l}} + \left(\frac{l\varepsilon}{2^{5}d(d+3)\log(8m/\delta_{0})\sqrt{2T\ln(1/\delta_{0})}}\right)^{2} \\
            &- \sqrt{\left(2 + \left(\frac{l\varepsilon}{2^{5}d(d+3)\log(8m/\delta_{0})\sqrt{2T\ln(1/\delta_{0})}}\right)^{2} \right)^{2} - 4}  = -2
        \end{aligned}
    \end{equation}}
    Therefore, thanks to Prop.~\ref{prop:privacy_per_batch}, we have that each update to the design matrix is $\left(\frac{\varepsilon\sqrt{l}}{\sqrt{T}}, \delta_{0} + \delta\right)$-DP. Therefore, using advanced composition yields the result.

    \subsection{Proof of Thm.~\ref{thm:regret.generic}}

    Let's now move on to the proof of the main theorem, Thm.~\ref{thm:regret.generic}.
    Let's note $l^{\star} = T/M_{S}$ where $M_{S}$ is the number of batch from the shuffler point of view, this parameter is given to the shuffler.
    Now let's consider a shuffler batch $j\leq M_{S}$, sent to the bandit algorithm, let's note then $q_{j} < j$ the last shuffler batch where Alg.~\ref{alg:complete.algo}
    has updated the estimate $\wt{\theta}$. Therefore, if Alg.~\ref{alg:complete.algo} decides to update the parameter $\wt{\theta}$ after receiving the data
    from the shuffler batch $j$, we have that:
\begin{align}
 \text{det} (\tilde{V}_{j}) \geq (1 + \eta)\text{det} (\tilde{V}_{q_{j}})
\end{align}

Let's consider any bandit batch $r$, between time $t_{r}+1$ and $t_{r+1}$ we can then decompose the interval $\{t_r +1, \dots, t_{r+1}\}$
into successive shuffler batches and we note the last of them $j_{r}$. That is to say, upon receiving the shuffler batch $j_{r}$ and $t_{j_r}$ the time step
at which this batch begins, Alg.~\ref{alg:complete.algo}
updates the parameter $\wt{\theta}$, so increasing the bandit batch from $r$ to $r+1$. Therefore, for all shuffler batch $j\leq j_r - 1$, we have that
$\det(\wt{V}_{j})\leq (1 + \eta)\text{det}(\wt{V}_{r})$
therefore for any vector $x\in\mathbb{R}^{d}$, $\left|\langle \theta^{\star} - \wt{\theta}_{r}, x\rangle\right| \leq \sqrt{1 + \eta}\beta_{r}\| x\|_{\wt{V}_{r}^{-1}}$ (see App.~$D$ in \citep{abbasi2011improved}).
In addition, for any time step $t$ during a batch $j$, $ \| x\|_{\wt{V}_{j}^{-1}}\leq \| x\|_{\wt{V}_{t}^{-1}}$ where $\wt{V}_{t}$ is the design matrix computed
with only data from the first $t$ time steps.
In addition for $t\in\{t_{j_{r}},\dots,  t_{r+1}\}$, we have that the norm $\|x\|_{\wt{V}_{r}^{-1}}$ can not be related to the norm of $\| x\|_{\wt{V}_{t}^{-1}}$ but we have that:
\begin{align}
  \left| \langle \theta^{\star} - \wt{\theta}_{r}, x\rangle\right| \leq \beta_{r}\| x\|_{\wt{V}_{r}^{-1}} \leq \frac{\beta_{r}\| x\|_{2}}{\sqrt{\lambda_{\min}(V_{r})}}
\end{align}
Therefore, we can write the regret as:
\begin{align}
  R_{T} &= \sum_{t=1}^{T} \langle \theta^{\star}, x_{t,a_{t}^{\star}} - x_{t,a_{t}} \rangle = \sum_{p=0}^{M_{R}} \sum_{t=t_{p}+1}^{t_{p+1}} \langle \theta^{\star}, x_{t,a_{t}^{\star}} - x_{t,a_{t}}\rangle
\end{align}
where $M_{R}$ is the number of batch of Alg.~\ref{alg:complete.algo}.
Using the reasoning above, we have:
\begin{align}
 R_{T} &\leq \sum_{p=0}^{M_{R} - 1}  \sum_{t=t_{p}+1}^{t_{j_{p}}} 2\beta_{p}\sqrt{1 + \eta} \| x_{t,a_{t}}\|_{\wt{V}_{t}^{-1}}
    + \sum_{t=t_{j_{p}}+1}^{t_{p+1}} 2\beta_{p}\| x_{t,a_{t}}\|_{\wt{V}_{t_{p}}^{-1}}\\
    &\leq 2\beta_{T}\sum_{t=1}^{T} \sqrt{1+\eta} \| x_{t,a_{t}}\|_{\wt{V}_{t}^{-1}} + \sum_{p=0}^{M_{R}-1} \frac{2\beta_{p}Ll^{\star}}{\sqrt{\lambda_{\min}(\wt{V}_{p})}}
\end{align}
where $l^{\star}$ is the length of a shuffler batch. In addition, the design matrix $\wt{V}_{p}$ is regularized to ensure that its minimum eigenvalues
grows at a rate of $\sqrt{t_{p}}$. Therefore we have that for any bandit algorithm batch $r$:
{\small\begin{align*}
  \frac{2\beta_{r}Ll^{\star}}{\sqrt{\lambda_{\min}(\wt{V}_{r})}}
  \leq 2Ll^{\star}\Bigg(\frac{\sigma \sqrt{2\log\left(\frac{2T}{\delta}\right) + d\log\left(3 + \frac{TL^{2}}{\lambda}\right)}}{\sqrt{\lambda_{r}}} + S\sqrt{3}& \\
  + \frac{d\left(\sqrt{t_{r}m\log\left(\frac{2}{\delta}\right)} + \frac{2\log(2/\delta)}{3} +  \frac{\sqrt{2}}{m}\sqrt{t_{r}\log\left(\frac{2}{\delta}\right)}\right)}{\lambda_{r}}\Bigg)&
\end{align*}}

Therefore using \citep{carpentier2020elliptical}, the regret can be bounded by:
{\small\begin{equation}
  \begin{aligned}
    R_{T} \leq 2\beta_{T}\sqrt{\left(1 + \eta\right)T\log\left(1 + \frac{T}{d\lambda}\right)} + \sum_{r=0}^{M_{R}-1} 2Ll^{\star}\Bigg(\frac{\sigma \sqrt{2\log\left(\frac{2T}{\delta}\right) + d\log\left(3 + \frac{TL^{2}}{\lambda}\right)}}{\sqrt{\lambda_{r}}}& \\
    + S\sqrt{3} + \frac{d\left(\sqrt{t_{r}m\log\left(\frac{2}{\delta}\right)} + \frac{2\log(2/\delta)}{3} +  \frac{\sqrt{2}}{m}\sqrt{t_{r}\log\left(\frac{2}{\delta}\right)}\right)}{\lambda_{r}}\Bigg)&
  \end{aligned}
\end{equation}}

We now proceed to bound each term individually. First, we have:
\begin{equation}
  \begin{aligned}
    \sum_{r=0}^{M_{R}-1} 2\sqrt{3}Ll^{\star}S \leq 2\sqrt{3}LSl^{\star}M_{R}
  \end{aligned}
\end{equation}
This is because the shuffler sends data on a fix length schedule. Also, we have:
{\small\begin{equation}
  \begin{aligned}
    \sum_{r=0}^{M_{R}-1} 2Ll^{\star}\frac{\sigma \sqrt{2\log\left(\frac{2T}{\delta}\right) + d\log\left(3 + \frac{TL^{2}}{\lambda}\right)}}{\sqrt{\lambda_{r}}}\leq \frac{ 2M_{R}Ll^{\star}\sigma}{\sqrt{\lambda}}\sqrt{2\log\left(\frac{2}{\delta}\right) + d\log\left(3 + \frac{TL^{2}}{\lambda}\right)}
  \end{aligned}
\end{equation}}
Finally,
{\small\begin{equation}
  \begin{aligned}
    \sum_{r=0}^{M_{R}-1} \frac{2Ll^{\star}d}{\lambda_{r}}\left(\sqrt{t_{r}m\log\left(\frac{2}{\delta}\right)} + \frac{2\log(2T/\delta)}{3} +  \frac{\sqrt{2}}{m}\sqrt{t_{r}\log\left(\frac{2}{\delta}\right)}\right) \leq
    \frac{4Ll^{\star}dM_{R}}{3}\log\left(\frac{2T}{\delta}\right)&\\+ \frac{\sqrt{2}Ll^{\star}dM_{R}m}{4}&
  \end{aligned}
\end{equation}}

Therefore with probability at least $1 - \delta$ the regret is bounded by:
\begin{equation}\label{eq:temp_regret}
  \begin{aligned}
    R_{T} \leq \underbrace{2\beta_{T}\sqrt{\left(1 + \eta\right)T\log\left(1 + \frac{T}{d\lambda}\right)}}_{:= \textcircled{a}} + \frac{ 2M_{R}Ll^{\star}\sigma}{\sqrt{\lambda}}\sqrt{2\log\left(\frac{2}{\delta}\right) + d\log\left(3 + \frac{TL^{2}}{\lambda}\right)}&
       \\+ \frac{4Ll^{\star}dM_{R}}{3}\log\left(\frac{2T}{\delta}\right) + \frac{\sqrt{2}Ll^{\star}dM_{R}m}{4} + 2\sqrt{3}LSl^{\star}M_{R}&
  \end{aligned}
\end{equation}

\textbf{Bounding \textcircled{a}.} Given the expression of $\beta_{T}$, we have that:
\begin{equation}
  \begin{aligned}
    \textcircled{a} \leq \sigma\sqrt{\left(8\log\left(\frac{2T}{\delta}\right) + d\log\left(3 + \frac{TL^{2}}{\lambda}\right)\right)\left(1 + \eta\right)T\log\left(1 + \frac{T}{d\lambda}\right)}&\\
    + \frac{2\sqrt{3}ST^{1/4}}{\sqrt{1-p}}\sqrt{\left(1 + \eta\right)T\log\left(1 + \frac{T}{d\lambda}\right)}& \\
    + \left(4dmT^{1/4} + \frac{8\log(2T/\delta)\sqrt{m}}{3} \right)\sqrt{\left(1 + \eta\right)T\log\left(1 + \frac{T}{d\lambda}\right)}&
  \end{aligned}
\end{equation}

Now, we are left with bounding the remaining of the right hand part of Eq.~\eqref{eq:temp_regret}. The first step to do so is to notice that the number of
bandit algorithm batch is bounded by roughly $\mathcal{O}\left(\log(T)\right)$, more precisely:
\begin{equation}
  \begin{aligned}
    M_{R} \leq 1 + \frac{d\log\left(\frac{L^{2}T}{d} + \frac{16\sqrt{T}\log\left(2T/\delta\right)}{(1 - p)}\right)}{\log(1 + \eta)}
  \end{aligned}
\end{equation}
In addition, if $l^{\star}$ satisfies Eq.~\eqref{eq:equation_privacy_dp} then we have that:
{\small\begin{equation}\label{eq:nb_batch}
    l^{\star} \leq \max\left\{\frac{8\log(2m/\delta_{0})}{p^{2}}, \frac{128\sqrt{2T\ln(2/\delta_{0})}d(d+1)\log(8m/\delta_{0})(1 - p)}{\varepsilon}, \frac{14\log(2m/\delta_{0})}{p} \right\}
\end{equation}}

In Eq.~\eqref{eq:nb_batch} we have that the regret is bounded with probability at least $1 - \delta$:
\begin{equation}
  \begin{aligned}
    R_{T} &\leq \frac{2\sqrt{3}(S + md)T^{3/4}}{\sqrt{1-p}}\sqrt{\left(1 + \eta\right)\log\left(1 + \frac{T}{d\lambda}\right)} \\
    &+ \frac{dLm}{\sqrt{\lambda}}\left(1 + \frac{d^{3/2}\log\left(\frac{L^{2}T}{d} + \frac{16\sqrt{T}\log\left(2T/\delta\right)}{(1 - p)}\right)^{3/2}}{\log(1 + \eta)} \right)\times\\
    &\times\max\Bigg\{\frac{14\log(8m/\delta_{0})}{p^{2}},
    \frac{128\sqrt{2T\ln\left(\frac{2}{\delta_{0}}\right)}d(d+1)\log\left(\frac{8m}{\delta_{0}}\right)(1-p)}{\varepsilon}\Bigg\}\\
  \end{aligned}
\end{equation}
Therefore, we can differentiate two different scenarios:
\begin{itemize}
  \item If $p^{2}(1 - p) \leq \frac{7T^{-1/2}  \varepsilon}{64\sqrt{2\ln(2/\delta_{0})}d(d+1)}$:
  \begin{align}
    R_{T} \leq &\frac{2\sqrt{3}(S + md)T^{3/4}}{\sqrt{1-p}}\sqrt{\left(1 + \eta\right)\log\left(1 + \frac{T}{d\lambda}\right)} \\
    &+ \frac{dLm}{\sqrt{\lambda}}\left(1 + \frac{d^{3/2}\log\left(\frac{L^{2}T}{d} + \frac{16\sqrt{T}\log\left(2T/\delta\right)}{(1 - p)}\right)^{3/2}}{\log(1 + \eta)} \right)\frac{14\log(8m/\delta_{0})}{p^{2}}
  \end{align}
  \item If $p^{2}(1 - p) \geq \frac{7T^{-1/2}  \varepsilon}{64\sqrt{2\ln(2/\delta_{0})}d(d+1)}$:
  {\small\begin{equation}\begin{aligned}
    R_{T} &\leq \frac{2\sqrt{3}(S + md)T^{3/4}}{\sqrt{1-p}}\sqrt{\left(1 + \eta\right)\log\left(1 + \frac{T}{d\lambda}\right)} \\
    &+ \frac{264}{\sqrt{\lambda}}\sqrt{2}d^{3}\log\left(\frac{8m}{\delta_{0}}\right)^{3/2}Lm\left(1 + \frac{d^{3/2}\log\left(\frac{L^{2}T}{d} + \frac{16\sqrt{T}\log\left(2T/\delta\right)}{(1 - p)}\right)^{3/2}}{\log(1 + \eta)} \right)\frac{\sqrt{T}(1-p)}{\varepsilon}
  \end{aligned}\end{equation}}
\end{itemize}

The last step now is to choose the parameter $\varepsilon_{0}$ to optimize the regret. Therefore, if $\varepsilon \leq \frac{1}{27T^{1/4}}$, so in a high privacy regime,
when choosing $ p = 1 - \varepsilon^{2/3}T^{1/6}$ we are in the second scenario above and:
\begin{align*}
  R_{T} \leq \frac{T^{2/3}}{\varepsilon^{1/3}}\Bigg[& \frac{264}{\sqrt{\lambda}}\sqrt{2}d^{3}\log\left(\frac{8m}{\delta_{0}}\right)^{3/2}Lm\left(1 + \frac{d^{3/2}\log\left(\frac{L^{2}T}{d} + \frac{16\sqrt{T}\log\left(2T/\delta\right)}{(1 - p)}\right)^{3/2}}{\log(1 + \eta)} \right) \\
  &+ 2\sqrt{3}(S + md)\sqrt{\left(1 + \eta\right)\log\left(1 + \frac{T}{d\lambda}\right)}\Bigg]
\end{align*}

\section{Regret with Scheduled Update Algorithm}\label{app:schedule_update_alg}

In this appendix, we present a bandit algorithm using a fixed schedule update instead of the
determinant based condition used in Alg.~\ref{alg:complete.algo}. The main consequence of
using a fixed batch bandit algorithm is a worse regret compared to Alg.~\ref{alg:complete.algo}.
That is a consequence of the inflated bonus needed by the use of the local randomizer algorithm
$\mathcal{M}_{\text{LDP}}$. Let's consider the batched algorithm described in Alg.~\ref{alg:fixed_batch_linucb}.

\begin{algorithm}[htp]
    \caption{FixedBatchedShuffling-LinUCB}
    \label{alg:fixed_batch_linucb}
    \SetAlgoLined
    \DontPrintSemicolon
    \KwIn{LDP parameter: $\varepsilon_{0}$, privacy parameter: $\varepsilon, \delta'$, regularization parameter:$\lambda$, context bound:$ L$,
    failure probability:$\delta$, low switching parameter: $\eta$, encoding parameter:$m$, dimension: $d$}
    Initialize $j^S = j^A = 0$, $\wt{\theta}_0 = 0$, $\wt{V}_0 = \lambda I_d$, $p = \frac{2}{\exp(\nicefrac{2\varepsilon_{0}}{(md(d+3)))} +1}$\;
    \For{$t = 0, 1, \ldots$} {
        User receives $\wt\theta_{j^A}$, $\wt{V}_{j^A}$ and $\beta_{j^A}$ and selects
        $
            a_t \in \argmax_{a\in [K]} \langle x_{t,a}, \wt{\theta}_{j^A}\rangle + \beta_{j^A} \| x_{t,a}\|_{\wt{V}_{j^A}^{-1}}
        $\;
        Observe reward $r_t$ and compute private statistics
        $(\wt{b}_t, \wt{w}_t )= \mathcal{M}_{\mathrm{LDP}}((x_{t,a_t},r_t), p, m, L)$\;
        \textbf{Communication with the shuffler}\;
        $B^S_{j^S} = B^S_{j^S} \cup (\wt{b}_t, \wt{w}_t)$\;
        \If{$|B^S_{j^S}| = l$}{
            Set $t_{j^S+1} = t$, compute a permutation $\sigma$ of $\llbracket t_{j^S} + 1, t_{j^S+1} \rrbracket$ and compute aggregate statistics
            \[
                \forall i\leq d, k\leq i, \qquad Z_{j^S, i} = \sum_{n=1}^{l}\sum_{q=1}^{m} \tilde{b}_{\sigma(n),i,q}
                ~~\text{ and } ~~
                U_{j^S, i, k} = \sum_{n=1}^l \sum_{q=1}^{m} \wt{w}_{\sigma(n),i,k,q}
            \]\;
            \vspace{-16pt}
            Set $U_{j^S, i, k} = U_{j^S, k, i}$, $B_{j^S+1} = \emptyset$ and $j^S = j^S + 1$\;
            \textbf{Communication with the bandit algorithm}\;
            Receives $(Z_{j^S-1}, U_{j^S-1})$ and compute candidate statistics
            \begin{align*}
                \wt{B}_{j^A+1}
                &= \wt{B}_{j^A+1} + \frac{Z_{j^S-1}}{m(1-p)} - \frac{l^{S}}{2(1-p)}\\
                \wt{V}_{j^A+1}
                &= \wt{V}_{j^A+1} + \frac{U_{j^S-1}}{m(1 - p)} - \frac{l^{S}}{2(1-p)} + 2(\lambda_{j^A+1} - \lambda_{j^A})I_d
            \end{align*}
            Compute $\theta_{j^A+1} = \frac{1}{L}\wt{V}_{j^A+1}^{-1}\wt{B}_{j^A+1}$\;
            Set $t_{j^A+1} = t$, $\beta_{j^A+1}$ and $\lambda_{j^A+1}$ as in Eq.~\eqref{eq:confidence_width} and Eq.~\eqref{eq:lambda_reg}\;
            Set $j^A = j^A+1$, $\wt{B}_{j^A+1} = \wt{B}_{j^A}$ and $\wt{V}_{j^A+1} = \wt{V}_{j^A}$\;
        }
    }
\end{algorithm}

In terms of privacy Alg.~\ref{alg:fixed_batch_linucb} enjoys the same guarantees as Alg.~\ref{alg:complete.algo}. For any $\delta\in(0,1)$, we have that with probability at least $1 - \delta$:
\begin{equation}
    \begin{aligned}
        R_{T}& = \sum_{t=1}^{T} \langle \theta^{\star}, x_{t,a_{t}^{\star}} - x_{t,a_{t}}\rangle = \sum_{j=1}^{M} \sum_{t=t_{j}+1}^{t_{j+1}} \langle \theta^{\star}, x_{t,a_{t}^{\star}} - x_{t,a_{t}}\rangle
    \end{aligned}
\end{equation}
But using Lem.$3$ in \citep{han2020sequential}, we have that for any batch $j$:
\begin{equation}
    \begin{aligned}
        \sum_{j=1}^{M}\sum_{t=t_{j}+1}^{t_{j+1}} \| x_{t,a_{t}}\|_{\wt{V}_{j}^{-1}} &\leq \sqrt{\frac{T}{M}}\sum_{j=1}^{M}\sqrt{\text{Tr}\left(\wt{V}_{j}^{-1}\sum_{t=t_{j}+1}^{t_{j+1}} x_{t,a_{t}}x_{t,a_{t}}^{\intercal}\right)}\\
        &\leq  \sqrt{\frac{10T}{M}}\log(T+1)\left( \sqrt{Md} + d\sqrt{\frac{T}{M}}\right)
    \end{aligned}
\end{equation}
where $M$ is the total number of batch. Therefore, the regret is bounded with high probability by:
\begin{equation}
    \begin{aligned}
        R_{T} \leq 2\beta_{M}\sqrt{\frac{10T}{M}}\log(T+1)\left( \sqrt{Md} + d\sqrt{\frac{T}{M}}\right) = \mathcal{O}\left(\frac{T^{3/4}}{\sqrt{1-p}} + \frac{T^{3/4}\sqrt{1-p}}{\varepsilon}\right)
    \end{aligned}
\end{equation}
Where we used the fact that $M$ is defined in Eq.~\eqref{eq:nb_batch}. Therefore, using a fixed schedule algorithm the trade-off highlighted in Thm.~\ref{thm:regret.generic} does not appear.

\end{appendix}
\end{document}